\documentclass[12pt]{article}

\usepackage{graphicx}
\usepackage{epsfig}
\usepackage{psfrag}
\usepackage{wrapfig}
\usepackage[all]{xy}
\usepackage{pa}

\usepackage[margin=1in]{geometry}
\usepackage{setspace}
\usepackage{longtable}
\usepackage{mathptmx}

\usepackage{fancyhdr}
\newcommand{\bi}{\begin{itemize}}
\newcommand{\ei}{\end{itemize}}

\newcommand{\specialcell}[2][l]{%
  \begin{tabular}[#1]{@{}l@{}}#2\end{tabular}}

\newcommand\wordcount{
    \immediate\write18{texcount -sum -1 \jobname.tex > 'count.txt'}
\input{count.txt}words}

\usepackage{natbib}

\usepackage{amsmath}
\usepackage{amsfonts}
\usepackage{amsthm}
\usepackage{bbm}
\usepackage{array}

\newtheorem{prop}{Proposition}

\newtheorem{assn}{Assumption}

\def\Var{{\rm Var}\,}
\def\E{{\rm E}\,}
\def\arg{{\rm arg}\,}

\def\N{{\rm N}\,}
\newcommand\independent{\protect\mathpalette{\protect\independenT}{\perp}}
\def\independenT#1#2{\mathrel{\rlap{$#1#2$}\mkern2mu{#1#2}}}

\newcommand{\nn}{\nonumber}

\title{Retrospective Causal Inference with Machine Learning Ensembles: An Application to Anti-Recidivism Policies in Colombia}

\author{Cyrus Samii, Laura Paler, and Sarah Zukerman Daly\thanks{Authors listed in reverse alphabetical order and are equal contributors to the project. Samii is Assistant Professor, Politics Department, New York University, 19 West 14th Street, New York, NY 10012 (email: cds2083@nyu.edu).  Paler is Assistant Professor, University of Pittsburgh, 4600 Wesley W. Posvar Hall, Pittsburgh, PA 15260 (email: lpaler@pitt.edu).  Daly is Assistant Professor, University of Notre Dame, 217 O'Shaughnessy Hall, Notre Dame, IN 46556 (email: sarahdaly@nd.edu). All replication materials are available at the {\it Political Analysis} Dataverse (article url: http://dx.doi.org/10.7910/DVN/QXCFO2).  We thank Carolina Serrano for excellent research assistance in Colombia and the team at {\it Fundaci\'on Ideas para la Paz} for their support in the data collection.   We also thank the Organization of American States {\it Misi\'on de Apoyo al Proceso de Paz} and the {\it Agencia Colombiana para la Reintegraci\'on} for their collaboration. Daly acknowledges funding from the Swedish Foreign Ministry.  For helpful discussions, the authors thank Michael Alvarez, two anonymous {\it Political Analysis} reviewers, Deniz Aksoy, Peter Aronow, Neal Beck, Matthew Blackwell, Drew Dimmery, Ryan Jablonski, Michael Peress, Fredrik Savje, Maya Sen, Teppei Yamamoto, Rodrigo Zarazaga, and seminar participants at APSA, EPSA, ESOC, MIT, MPSA, NYU, and University of Rochester.}}

\date{\today}

\begin{document}
\maketitle

\begin{center}
{\large Forthcoming at {\it Political Analysis}.}
\end{center}

\thispagestyle{empty}

\clearpage

\thispagestyle{empty}
\begin{center}
\doublespacing
{\LARGE Retrospective Causal Inference with Machine Learning Ensembles: An Application to Anti-Recidivism Policies in Colombia}
\end{center}
\vspace{.5in}

\begin{abstract}
We present new methods to estimate causal effects retrospectively from micro data with the assistance of a machine learning ensemble.  This approach overcomes two important limitations in conventional methods like regression modeling or matching: (i) ambiguity about the pertinent retrospective counterfactuals and (ii) potential misspecification, overfitting, and otherwise bias-prone or inefficient use of a large identifying covariate set in the estimation of causal effects.  Our method targets the analysis toward a well defined ``retrospective intervention effect'' (RIE) based on hypothetical population interventions and applies a machine learning ensemble that allows data to guide us, in a controlled fashion, on how to use a large identifying covariate set.  We illustrate with an analysis of policy options for reducing ex-combatant recidivism in Colombia.
\end{abstract}

\vspace{1em}\\
\noindent Word count:  9,756 (including main text, captions, and references).

\clearpage
\doublespacing
\setcounter{page}{1}
\section{Introduction}

Retrospective causal studies are essential in the social sciences but they present acute challenges.  They are essential insofar as for some important causal questions there are often no feasible alternatives to a retrospective analysis. Such situations include studies of rare outcomes or outcomes that take many years to come about, such as violence or institutional changes.  Adequately powered prospective studies, whether in the form of a randomized experiment or not, may take too long and be too logistically difficult to be practical or may prove unethical.  

Retrospective studies present acute challenges because they try to make causal inferences about the effects of policies, exposures, or processes that were beyond the control of analysts. This introduces problems of endogeneity and confounding.  Moreover, generating results that can inform policy requires estimates that are relevant for one's target population, but sources of quasi-random variation (e.g., instrumental variables or discontinuities) may be too specific in the subpopulations to which they apply to meet these needs.  The relevant counterfactual comparisons may not be obvious either. 

We draw on new methods from epidemiology and apply a machine learning approach to overcome these challenges \citep{vanderlaan2011-targeted}.  Our approach makes use of familiar ``conditional independence'' assumptions, however we do so in a way that circumvents problems that arise in simpler uses of regression, matching, or propensity scores \citep{angrist_pischke09}.\footnote{We define conditional independence formally below. The idea is that we can identify the set of confounding factors and ``condition'' on them, thereby removing the confounding covariation.}  Specifically, we use a very large number of covariate control variables and a machine learning ensemble.  Using a very large number of covariates allows us to make conditional independence more believable, which in principle also moves us safely past concerns about ``bias amplification'' \citep{myers-etal2011-bias-amplification}.\footnote{Bias amplification can occur when omitted variables confound estimates of a causal effect and one incorporates additional covariates that purge substantial variation from the treatment variables but fail to purge variation from the outcome variables \citep{pearl2010-bias-ampliyfing}.  Risk of bias amplification depends on the specificities of a given data set. \citet{myers-etal2011-bias-amplification} find empirically that  such biases tend not to be a major concern in epidemiological applications with reasonable sets of control variables.}  But having such a rich covariate set raises questions about how to properly employ the covariates.  We face the daunting task of having to choose from among the vast possibilities for terms (e.g., squared, cubed) or interactions to include in a model.  We use a machine learning ensemble that lets the data guide us, in a controlled fashion, in using an identifying covariate set.  We use a simulation experiment to show how a machine learning ensemble is more robust than conventional methods in extracting identifying variation from irregular functional relationships in a noisy covariate space.

To obtain causal estimates that properly inform realistic policy options, we define our counterfactuals in terms of substantively motivated ``retrospective intervention effects'' for the target population.  The retrospective intervention effect (RIE) establishes a compelling counterfactual comparison that incorporates different types of information than alternative estimands such as the average treatment effect (ATE), average effect of the treatment on the treated (ATT), or the average effect of the treatment on the controls (ATC). (We provide a formal characterization of the differences below.)  Consider an analysis of the effects of employment on criminality.  The RIE compares what actually occurred in the population to a counterfactual where everyone in the population is ensured to be employed.  In contrast, the ATE would estimate how criminality differs when everyone is employed versus when no one is employed, an unrealistic population counterfactual.  The ATT and ATC are less unrealistic than the ATE, in that they compare how things would change were we to intervene on the employment status among those with and without jobs, respectively.  But they cannot speak to the importance of such interventions in the population because they do incorporate pre-intervention levels of employment.  Taking pre-existing rates of employment into account is especially important if one wanted to compare an employment intervention to, say, cognitive behavioral therapy, for reducing overall crime rates.  That said, in some cases estimands other than the RIE may be preferable---it would depend on the goals of the analysis.  The ensemble methods that we apply here could be used for other estimands.

This paper contributes to the political methodology literature on causal inference in two ways.  First, we offer a didactic presentation of how one can apply the power of machine learning ensembles to causal inference and policy analysis problems.  In doing so we demonstrate how causal inference problems are extensions of ensemble prediction problems, something with which political scientists are already somewhat familiar \citep{montgomery-etal2012-ebma}.  Second, we demonstrate the use of hypothetical interventions as a way to target the analysis toward a substantively meaningful counterfactual comparison that yields the ``retrospective intervention effect.'' Our application to retrospective studies extends the existing literature on machine learning for causal inference, which includes work on characterizing heterogenous treatment effects \citep{athey-imbens2015-het-effects, grimmer-etal2014-het-effects, green_kern2012_modeling_het, imai_ratkovic2012_estimating_heterogeneity, imai_strauss2011_optimal_gotv}, locating subpopulations within which conditional ignorability holds \citep{ratkovic2014-svm-balance},  and non-parametrically estimating counterfactual response surfaces \citep{hill2011-bart}.  Third, the high-dimensional propensity score and reweighting methods that we use are readily applicable to other types of reweighting methods, such as for dynamic treatment regimes \citep{blackwell2013-dynamic}.

We begin by establishing the inferential setting, and then we discuss potential perils in standard practice for retrospective studies.  Next, we develop an approach to identification of causal effects based on hypothetical interventions.  Following that, we discuss estimation, practical implementation, and inference.  We apply the methods to an illustrative case study that evaluates policy options for reducing recidivism among ex-combatants in Colombia. A conclusion draws out implications and ideas for further research.

\section{Setting}

Our approach in this paper is based on the innovations of \citet{hubbard-vanderlaan2008-pop-int}, \citet{vanderlaan2011-targeted}, and \citet{young-etal2009-exploratory-epi}, and so we adopt their notation so as to allow readers to refer back to these reference works easily. We start with a target population and then obtain from it a random sample of observations.\footnote{A subsequent section deals with questions associated with unequal probability sampling or cluster sampling.} The observations consist of treatment variables denoted as the vector of random  variables $A=(A_1,...,A_j,...,A_J)'$, covariates denoted as the vector of random variable $W=(W_1,...,W_p,...,W_P)'$, and an outcome variable $Y$.  These observations are defined collectively by the random vector $O=(W, A, Y)'$ that is governed in the target population by some probability distribution, $P_0$.   The task is to estimate the average causal effects of components of $A$ for our target population.  An  arbitrary component of the treatment vector $A$ is labeled as $A_j$, the complement of elements in $A$ is labeled as $A_{-j}$, and the support for $A_j$ denoted as $\mathcal{A}_j$.  

The causal structure is assumed to follow the graph depicted in Figure \ref{fig:ordering} \citep{pearl-2009-causality}.  We have circled the elements of $A$ to highlight our interest in estimating causal effects for the components of that vector.  The causal graph indicates two sources of confounding, originating in $W$ and $U$, with the variable $U$ standing in to characterize any unobserved determinants of the elements of $A$.  The assumptions embedded in this graph indicate that for estimating the effect of $A_j$, confounding originating in $W$ can be blocked by conditioning on $W$, while confounding originating in $U$ can be blocked by conditioning on $A_{-j}$.  An important assumption that this graph encodes is aside from the dependencies due to $U$ and $W$, there are no direct causal relationships between the elements of $A$.  These are substantive assumptions about the causal structure.\footnote{If they are wrong, the analysis will not generally yield unbiased or consistent estimates of causal effects.  In an applied setting, one would want to check robustness of one's estimates to a variety of assumptions about the causal graph. For example, one would want to check to see whether estimates change if one assumes that some elements of $A$ are causally dependent on others.  Under such alternative assumptions, one would set up the analysis in ways that avoid post-treatment bias by including in the set of covariate controls only the elements of $A_{-j}$ that are not causally dependent on $A_j$ \citep{king_zeng06_extreme, rosenbaum84}.  Once that is done, the analysis would proceed as we describe below.  Our primary interest in this paper is to elaborate methods given a causal graph, and so to save space we do not conduct such robustness checks here.}

\begin{figure}[t]
\begin{center}
\includegraphics[width=.5\textwidth]{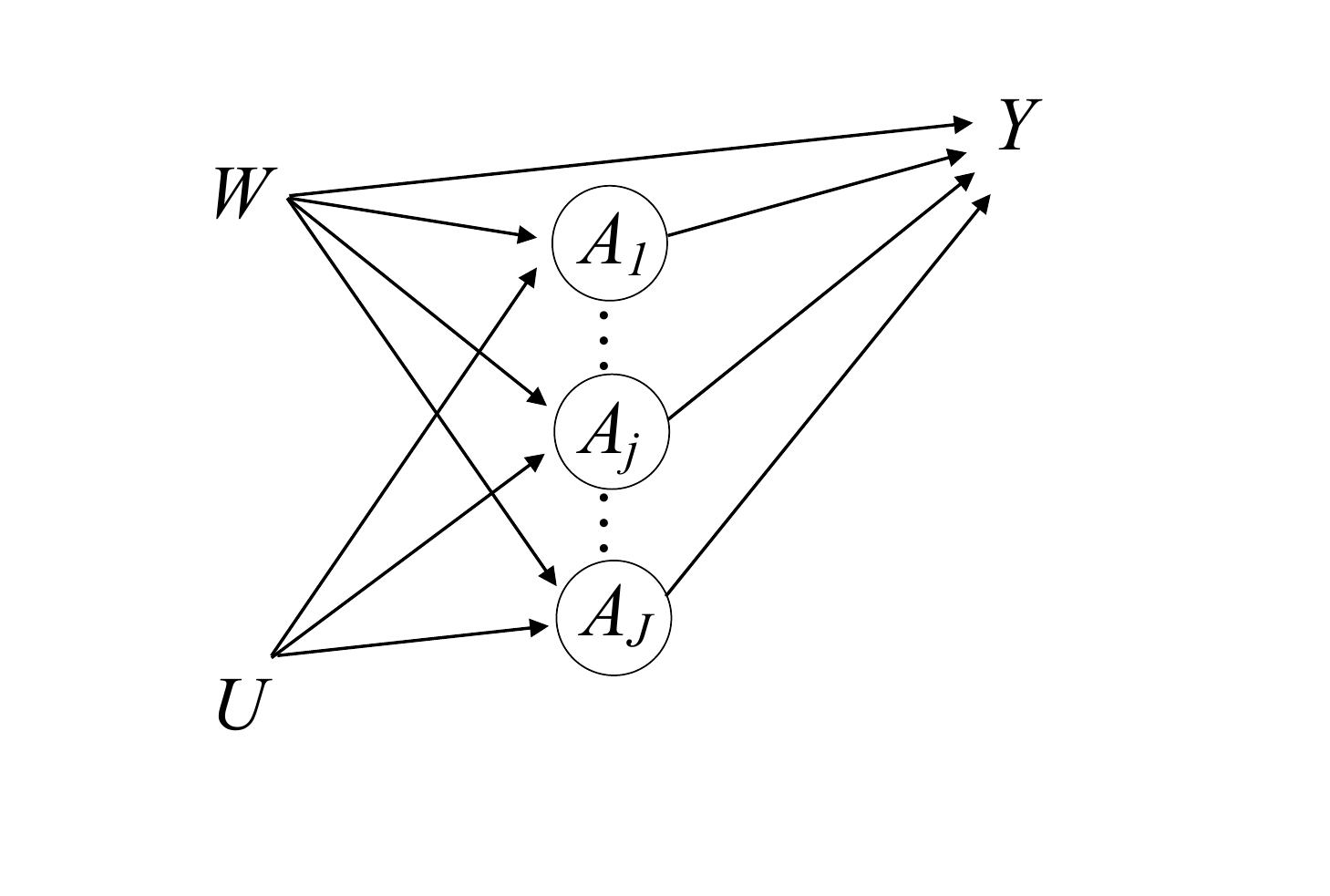}
\end{center}
\caption{\label{fig:ordering} Assumed causal graph, showing confounding in $W$ for the effect of $A_j$ can be blocked by conditioning on $(W, A_{-j})$, and then confounding originating in $U$ that can be blocked by conditioning on $A_{-j}$.}
\end{figure}

Using the ``potential outcomes'' notation to define causal effects \citep{holland86, rubin78, sekhon09_opiates}, we can write the outcome that would be observed if treatments $(A_1,...,A_J)$ were set to $(a_1,...,a_J)$ as follows:
$$
Y(a) = Y(a_1,...,a_J),
$$
with $a \in \prod_{j=1}^J \mathcal{A}_j \equiv \mathcal{A}$.  Thus, potential outcomes depend on the combinations of treatments a unit receives, with these combinations denoted by the vector $a$.  For an arbitrary unit $i$ in our target population, the causal effect of fixing $A_{ji} = a$ versus $A_{ji} = a'$ is defined as,
$$
\tau_{ji}(a,a') = Y_{i}(a, A_{-j}) - Y_{i}(a', A_{-j}),
$$
where the introduction of the $i$ subscripts highlights our focus on possible heterogeneity in these effects across units.  Define $\tau_{j}(a,a') = \E[\tau_{ji}(a,a')]$, the average causal effect with the average taken over the units indexed by $i$.  This target quantity, $\tau_{j}(a,a')$, is non-parametrically identified under the so-called conditional independence assumption (\citealp{angrist_pischke09}, pp. 52-59; \citealp{imai_vandyk04_genpscore}; \citealp{imbens2004-ate}; \citealp{imbens_wooldridge09}):
$$
A_j \independent (Y_{i}(a, A_{-ji}), Y_{i}(a', A_{-ji}))' | (A_{-ji},W)'.
$$
Figure \ref{fig:ordering} implies this assumption (although other graphs could also be drawn under this assumption too).  Here, $A_{-j}$ and $W$ form a conditioning vector that blocks sources of confounding variation (or ``back door paths'',  \citealp{pearl-2009-causality}, pp. 16-18, 78-81) in the relationship between $A_j$ and our potential outcomes, $Y_{i}(a, A_{-ji})$ and $Y_{i}(a', A_{-ji})$.

\section{Perils of standard practice}

Conditional independence of the treatments offers the promise of being able to identify causal effects. But one still faces the challenge of operationalizing conditional independence.  \citet{imbens2004-ate} reviews general approaches rooted in either (i) propensity scores and a focus on the ``assignment mechanism'' that determines the relationship between covariates, $(A_{-j},W)'$, and the causal factor of interest, $A_j$, or (ii) response surface modeling and a focus on outcome data generating processes that relate covariates, $(A_{-j},W)'$, to outcomes, $(Y(a, A_{-j}), Y(a', A_{-j}))'$.  As Imbens shows, accounting for either assignment or response is sufficient for identifying a causal effect under the conditional independence assumption.  Analysts have put forward various arguments for whether it is preferable to emphasize assignment \citep{rosenbaum_rubin83_pscore, rubin2008-design-trumps}, response surfaces \citep{hill2011-bart, pearl2010-bias-ampliyfing}, or a combination of the two in the construction of ``doubly robust'' estimators \citep{bang-robins2005-doubly-robust, robins-rotnizky1995-semipar-eff}.

Regression modeling, the workhorse method in the social sciences, can be variously conceptualized as following either approach.  Following \citet[pp. 52-59]{angrist_pischke09}, suppose effects are homogenous such that $\tau_{ji}(a,a') = \tau_{j}(a,a')$ for all units, and that one defines the conditioning vector $X_i \equiv (A_{-ji},W_i)'$  in a regression model of the form,
$$
Y_{i} = \alpha + \beta A_{ji} + X_i \gamma + \epsilon_i.
$$
We suppose the error term, $\epsilon_i$, equals the ordinary least squares residual from the regression of $Y_{i} - \alpha - \beta A_{ji}$ on $X_i$ when this regression is carried out on the full population for which one wants to make inference.  Then, so long as the control vector specification in $X_i$ is adequate to ensure that the linearity assumption holds---that is  $\E[Y_{i} - \alpha - \beta A_{ji}|X_i] = X_i \gamma$ holds---the ordinary least squares (OLS) estimate of $\beta$ is consistent for the homogenous effect, $\tau_{j}(a,a')$ \citep[pp. 57-59]{angrist_pischke09}.  This is in essence a response modeling approach.  In contrast, \citet{angrist_krueger99_handbook} 
and \citet{aronow_samii2016_whose_effect} 
develop the case where the control function, $X_i \gamma$, models the assignment process.  In this case, the  homogenous effects assumption again implies that the OLS estimator for $\beta$ is consistent for $\tau_{j}(a,a')$.  

These two assumptions---homogenous effects and correct  specification for the control vector, $X$---are unrealistic in many applied settings, making the naive use of linear regression a problematic tool for exploiting conditional independence of the treatment.  Furthermore, it would be heroic to presume that all relevant heterogeneity could be modeled.  The linearity assumption is especially vexing when conditional independence of the treatment requires a large covariate set, as this introduces a bewildering array of possible higher order terms and interactions that one must decide on including or excluding.  If either homogenous effects or correct linear specification fails to hold, causal effects estimated with linear regression may fail to characterize the average causal effects for the target population.  First, even if linearity in $X$ holds but effects are {\it heterogeneous}, then the OLS estimator recovers a distorted estimate of the average causal effect.  The distortions are based on an implicit weighting that linear regression produces based on the conditional variance of $A_j$ (\citealp{angrist_pischke09}, p. 75; \citealp{angrist_krueger99_handbook}; \citealp{aronow_samii2016_whose_effect}).\footnote{While the key results in these papers are developed with respect to ordinary least squares regression, as \citet{aronow_samii2016_whose_effect} show, the very same results apply in the first order to estimates for generalized linear models such as logit, probit, and so on.}  
Second, when the specification based on $X$ is wrong, residual confounding may remain and bias the results.  Beyond these risks of getting it wrong, there is also the question of researcher discretion through which terms in $X$ may be manipulated to produce ``desirable'' results \citep{king_zeng06_extreme}.  

Direct covariate matching is an alternative to regression and it relieves the analyst from some of the modeling burdens necessary with regression \citep{ho_etal07_matching}.  Nevertheless, direct covariate matching becomes difficult when the covariate space is large.  When that is the case, one is forced to apply some method of characterizing distance in the covariate space in order either to identify ``nearest neighbors'' or, in kernel matching, generate kernel-weighted approximations of counterfactual outcomes \citep{imbens_wooldridge09}.  Generally speaking, distance metrics for direct covariate matching convey no optimality criteria with respect to bias minimization.  Matching on propensity scores \citep{rosenbaum_rubin83_pscore} or prognostic scores \citep{hansen08-prognostic} can resolve such dimensionality problems and in a manner that is targeted toward bias minimization, but in practice one is left with the task of determining a specification for the propensity or prognostic scores.  When the covariate space is very large, similar challenges make it difficult to use other ``direct balancing'' methods such as entropy balancing \citep{hainmueller11_entropy_balancing}.

The idea that we pursue is that a machine learning approach might allow us to sift through the information content in a large covariate set to target bias minimization in an efficient manner.  Machine learning methods are distinguished from other statistical methods in their emphasis on ``regularization,'' which is the use of penalties for model complexity (\citealp{bickel-li2006-regularization}; \citealp{hastie-etal2009-elements}, p. 34), as well as processes of tuning models so as to minimize cross-validated prediction error.  Our machine learning ensemble targets prediction error for propensity scores.  By combining regularization and cross-validation, the ensemble is built to wade through the noisy variation in a large covariate set and extract meaningful predictive covariate variation.  Because we are predicting propensity scores, this predictive variation is also variation that provides the basis for causal identification.  As \citet{vanderlaan2011-targeted} show, one could also use machine learning in a response-surface modeling approach. However using propensity scores allows for one round of machine learning that can then be used to estimate effects on a variety of outcomes, whereas a response modeling approach would require a separate machine learning step for each outcome.  \citet{busso_etal14_reweighting_vs_matching} show that when covariate distributions have good overlap over the treatment values, estimation using inverse propensity score weights exhibits favorable efficiency properties.  Below, we use a simulation study to illustrate these points.

\section{Defining retrospective intervention effects}

The first step of our approach is to define coherent causal quantities given that effects are possibly heterogeneous and nonlinear.  We do so through the definition of the ``retrospective intervention effect'' (RIE). %which requires proposing hypothetical population interventions.  The RIE compares what actually occurred in the population to what would have happened had the population been subject to a hypothetical intervention.  The hypothetical intervention may vary treatment by individual (e.g., the intervention might use scores to determine who gets treated and who is left alone). 
%
%Suppose the observed data are sample draws of a random vector, $O = (W,A,Y)$, as defined above, and the causal structure remains as in Figure \ref{fig:ordering}.  For the moment, we also take $O$ to be i.i.d. (e.g., the result of an equal probability sample from some population distribution), although relaxing that assumption does not change the basic identification results.  We will attend to the implications of non-i.i.d. data arising from cluster and non-equal probability sampling below.
Following  \citet{hubbard-vanderlaan2008-pop-int}, we consider hypothetical population interventions on the components of $A$.  Such hypothetical interventions are conceptualized as taking a treatment, say $A_j$, and imagining a manipulation that changes $A_j = a_j$ to $A_j = a_j'$.  Defining hypothetical interventions has two methodological benefits.  First, it allows us to define clear causal estimands under effects that vary not only from unit-to-unit, but also over different values of the underlying causal factors (e.g., non-linear or threshold effects).  Second, we can define potential interventions in a manner that takes into account real-world options and therefore establish estimands that are directly relevant for policy analysis \citep[pp. 54-58]{manski-1995-identification-book}. Different hypothetical interventions can be compared to each other in terms of their costs and estimated effects so as to come up with a ranking of the kinds of manipulations that are most promising from a practical perspective.

Our goal is to estimate, retrospectively, the effects of hypothetical interventions  associated with each component of $A$ on the outcome distribution for the population.  That is, we seek to estimate the difference between what has {\it actually happened} against a counter-factual of {\it what would have happened} had there been an intervention on variable $A_j$. The way that one defines hypothetical interventions depends on the types of practical questions that one wants to answer. Consider an intervention on $A_j$ defined as fixing $A_j = \underline{a}_{j}$ for all members of the population.  If $\underline{a}_{j}$ were the minimum value of $A_j$, for example, then the retrospective intervention effect would be equivalent to what epidemiologists refer to as the ``attributable risk'' \citep[63]{rothman-etal2008-epi}, which measures the average consequence of the observed level of $A_j$ relative to a counterfactual of $A_j$ being kept to its minimum throughout the population.

Another type of hypothetical intervention is one that manipulates values of a  continuous treatment, but does so in a manner that varies depending on individuals' realized values of the treatment variable.  For example, suppose the causal factor of interest is income.  We could define an intervention that ensures that all individuals have some minimum level of income, $\underline{c}$.  Then, we apply this intervention to all individuals, in which case we would be changing the incomes for all individuals with incomes less than $\underline{c}$ to be, counterfactually, $\underline{c}$.  For individuals with incomes above $\underline{c}$, the intervention would have no effect and so their incomes would remain as observed.

For outcome $Y$, define the retrospective intervention effect (RIE) for $A_j$ and intervention value $\underline{a}_j$ as,
$$
\psi_{j} = \underbrace{\E[Y(\underline{a}_j,A_{-j})]}_{\text{counterfactual mean}} - \underbrace{\E[Y]}_{\text{observed mean}},
$$
where $A_{-j}$ refers to elements of $A$ other than $A_j$.  %The RIE differs from the average treatment effects that one commonly sees in the causal inference literature \citep{imbens_wooldridge09} in that it compares a counterfactual defined on the basis of a hypothetical intervention to the observed mean.  
The RIE has a direct relationship to the average effect of the treatment on the treated (ATT) or average effect of the treatment on controls (ATC) depending on the nature of the intervention that one wants to study. To see this, suppose a binary intervention variable, $A_j = 0,1$ and that the intervention of interest is one that sets $A_j=0$ (e.g., it is an intervention that  protects individuals from a harmful exposure). Then,
\begin{align}
\psi_j & = \E[Y(0, A_{-j})] - E[Y] \nonumber \\
& = \left\{ \E[Y(0, A_{-j})|A_j = 0]\Pr[A_j = 0] + \E[Y(0, A_{-j})|A_j = 1]\Pr[A_j = 1] \right\} \nonumber \\
&  \hspace{2em} -  \left\{ \E[Y(0, A_{-j})|A_j = 0]\Pr[A_j = 0] + \E[Y(1, A_{-j})|A_j = 1]\Pr[A_j = 1] \right\}  \nonumber \\
& = \left\{ \E[Y(0, A_{-j})|A_j = 1] - \E[Y(1, A_{-j})|A_j = 1]\right\} \Pr[A_j = 1]. \nonumber
\end{align}
Now note that ATT for $A_i$ is defined as, 
$$
ATT \equiv \E[Y(1, A_{-j})|A_j = 1] - \E[Y(0, A_{-j})|A_j = 1] = -\frac{\psi_j}{\Pr[A_j = 1]}.
$$
For this intervention, the RIE has a close relationship to the ATT.  A similar decomposition would follow for the ATC if we defined the intervention of interest as one that sets $A_i = 0$.  What is important to note here is how the RIE depends on the nature of the intervention that is being considered and how it incorporates information on the proportion of units that would be affected by the intervention.

We set the RIE as our target for a few reasons.  First, it compares a policy-relevant counterfactual to what has actually happened.  It allows us to answer the question of whether it would have been ``worth it'' to have pursued various interventions, using observed reality as a benchmark.  We feel that this provides a very coherent way to assess the policy relevance of different causal factors.  It takes as a starting place considerations of whether a causal factor could be manipulated, to what extent, and at what cost, and then quantifies the effects.  Second, the nature of the comparison limits the number of ``unknowns'' that we need to address in the analysis while still allowing us to address policy relevant questions clearly.  Given our sampling design, the observed outcome mean ($\E[Y]$) is identifiable from our data with no special assumptions.  Our analytical task is merely to characterize the counterfactual mean ($\E[Y(\underline{a}_j,A_{-j})]$). This makes for a more tractable analysis than would be the case, say, of comparing two counterfactual means when estimating an ATE (e.g., comparing two hypothetical interventions against each other). Our approach is consistent with recommendations of \citet[Ch. 3]{manski-1995-identification-book}, who proposes that one should target causal estimands depending on the data at hand, the policy questions one wants to answer, and the treatment regimes that different policies might imply.

\section{Identification and estimation}

The identification of the RIE, $\psi_{j}$, requires the following assumptions.

\begin{assn} 
$A = a$ implies $Y = Y(a)$.
\end{assn} 
\citet{vanderlaan2011-targeted} and \citet{vanderwheele2009-consistency} call this the ``consistency'' assumption,  and it also forms the basis of what \citep{rubin1990-sutva} calls the stable unit treatment value assumption or ``SUTVA.''  It means that when we observe $A=a$ for a unit, we are sure to observe the corresponding potential outcome $Y = Y(a)$ for that unit, and this is true regardless of what we observe in other units.\footnote{This usage of the word ``consistency'' should not be confused with its other meaning with reference to the asymptotic convergence of an estimator to a target parameter.}  This assumption would be violated in situations of ``interference,'' where units' outcomes are affected by the treatment status of other units \citep{cox1958}.  In such cases, one could try to redefine units of analysis to some higher level of aggregation such that Assumption 1 is plausible. 

\begin{assn}
For any $\underline{a}_j$ considered in the analysis,
$$
Y(\underline{a}_j, A_{-j}) \independent A_j|(W, A_{-j}).
$$
\end{assn}
This conditional independence assumption requires that conditioning on $W$ and $A_{-j}$ breaks any dependence between the realized value of the particular exposure, $A_j$, and potential outcomes when $A_j=\underline{a}_j$. The causal graph in Figure \ref{fig:ordering} establishes that this assumption allows for causal identification.  This assumption would be violated if the true data generating process departed from Figure \ref{fig:ordering} in particular ways, including causal relations between the elements of $A$, or the existence of other unmeasured confounders that causally determined $Y$ and elements of $A$.  In such cases, one would either have to limit the analysis to elements of $A$ for which Figure \ref{fig:ordering} is valid, or collect additional data to restore the causal dependence and independence assumptions encoded by Figure \ref{fig:ordering}.

\begin{assn}
For all $\underline{a}_j$ considered in the analysis, $\Pr[A_j = \underline{a}_j|W, A_{-j}] > b$ for some $b > 0$.
\end{assn}
This ``positivity'' or ``covariate overlap'' assumption allows us to construct the counterfactual distribution of potential outcomes under the intervention, $A_j = \underline{a}_j$, using the set of observations for which $A_j = \underline{a}_j$ in the sample \citep{petersen-etal2011-positivity}.  This assumption is necessary to identify the population-level counterfactual and therefore to obtain the population-level RIE.  If it does not hold, then identification would be restricted to the subpopulation with values of $W$ and $A_{-j}$ for which Assumption 3 does hold.

These assumptions above identify the population level counterfactual mean, $\E[Y(\underline{a}_j,A_{-j})]$, as follows:
\begin{align}
\E[Y(\underline{a}_j,A_{-j})] & = \E[\E[Y(\underline{a}_j,A_{-j})|W, A_{-j}]] \nn \\
& = \E[\E[Y(\underline{a}_j,A_{-j})|W, A_{-j}, A_j=\underline{a}_j]] \nn \\
& = \E[\E[Y|W, A_{-j}, A_j=\underline{a}_j]]\nn,
\end{align}
where the last term can be estimated using the observed $Y$ outcomes for units with $A_j = \underline{a}_j$.  The outer expectation is what is key: in constructing this counterfactual population average, one needs to weight the contributions of the $(W, A_{-j})$-specific $Y$ means in a manner that corresponds to the distribution of $(W, A_{-j})$ in the population.  The inverse-propensity score weighted approach that we explain below reweights the subpopulation of units with $A_j = \underline{a}_j$ such that it resembles the target population.

We use this identification result to construct an inverse-propensity score weighted (IPW) estimator of the RIE:
\begin{equation}
\hat \psi^{IPW}_{j} = \frac{1}{N}\sum_{i=1}^n\left( \frac{I(A_{ji}=\underline{a}_j)}{\hat{g}_j(\underline{a}_j|W_i,A_{-ji})}Y_{i}\right) - \bar{Y} \label{eq:rie-ipw}
\end{equation}
where $N$ is the sample size and $\hat{g}_j(\underline{a}_j|W_i,A_{-ji})$ is a consistent estimator for $\Pr[A_j=\underline{a}_j|W_i,A_{-ji}]$.  In essence, we take a weighted average of the outcomes of those units for which $A_j=\underline{a}_j$ without an intervention, where the weighting essentially expands each of these units' outcome contributions so that it proxies for the appropriate share of the population with $A_j \ne \underline{a}_j$. For example, if the intervention is the establishment of the income floor, $\underline{c}$, then the share of the population for which $A_j \ne \underline{a}_j$ is the share with incomes below $\underline{c}$. To construct the counterfactual mean under the income floor intervention, we expansion-weight certain individuals with incomes above $\underline{c}$ to approximate contributions from those with incomes below $\underline{c}$.  The way that we identify individuals to expansion-weight is through their covariate profiles, $(W_,A_{-j})$.  In the supplementary materials we show that under mild conditions on the data, $\hat \psi^{IPW}_{j}$ is consistent for $\psi_{j}$ and we can construct conservative confidence intervals.  In our application below we also account for unequal-probability cluster sampling.

\section{Ensemble methods for propensity scores\label{sec:ensemble}}

We do not typically know the functional form for the propensity score, $g_j(\underline{a}_j|W_i,A_{-ji})$, and so we use a machine learning ensemble method known as ``super learning'' to approximate such knowledge \citep{polley-etal2011-superlearning, vanderlaan-etal2007-superlearner}.  The super learner methodology is very similar to ensemble Bayesian model averaging (EBMA) discussed by \citet{montgomery-etal2012-ebma}.  Both super learning and EBMA compute a weighted average of the output of an ensemble of models, where each model is weighted on the basis of some loss criterion, and loss scores for the members of the ensemble are generated using cross-validation.  Ensemble methods relieve the analyst from having to make arbitrary choices about what estimation method to use and what specifications to fix for a given estimation method.  Rather, the analyst is free to consider a variety of estimation methods (linear regression methods, tree-based methods, etc.). Then, one uses cross-validation to determine the loss (e.g., the mean square prediction error) associated with each method.  Finally, the loss value associated with each method is used to determine the weight given to predictions from each method in the analysis.  Using cross-validated loss helps to minimize risks associated with over-fitting.  

To obtain our super learner ensemble estimate of the propensity score, we first obtain propensity score estimates from a set of candidate estimation algorithms. Then, to construct the ensemble estimate, we take a weighted average of estimates from the candidate algorithms.  The weighting is done in a way that minimizes the expected mean squared error (MSE).  

Formally, we have a set of candidate estimation algorithms indexed by $c=1,...,C$.  For each candidate algorithm we have an estimator, $\hat g_j^{c}(\cdot)$, that we fit to the data from each of the cross-validation splits, which are indexed by $v = 1,...,V$. The cross-validation splits are constructed by randomly partitioning the data into $V$ subsets; then each split consists of an estimation subsample of size $N-(N/V)$ and then a hold out samples of size $N_v = N/V$.  For each candidate algorithm, we fit the model on the estimation subsample to obtain $\hat g_j^{c,v}(\cdot)$, and then we generate predictions to the units in the hold out sample.  From that, the average MSE over the cross validation splits for candidate algorithm $c$ is
\begin{align}
\ell^{c}_j & = \frac{1}{V} \sum_{v=1}^{V}  \frac{1}{N_v}\sum_{i=1}^{N_v} [I(A_{ji} = \underbar{a}_j) - \hat g_j^{c,v}(\underline{a}_j|W_i,A_{-ji})]^2 \nonumber \\
& = \frac{1}{N} \sum_{i=1}^{N}  [I(A_{ji} = \underbar{a}_j) - \hat g_j^{c,v(i)}(\underline{a}_j|W_i,A_{-ji})]^2, \nonumber
\end{align}
where $v(i)$ indexes the cross validation split that contains unit $i$ in the hold-out sample.  The last line shows that each unit  receives a set of predicted values generated by each algorithm from when the unit was in a hold-out sample.  Moving from a single candidate algorithm to the ensemble, we seek the minimum-MSE weighted average of candidate algorithm estimates, which we obtain by solving for the ensemble weights as
\begin{align}
(w_j^{1*},...,w_j^{C*}) & ={\arg\min}_{(w_j^1,...,w_j^C)} \frac{1}{N} \sum_{i=1}^{N}  \left[I(A_{ji} = \underbar{a}_j) - \sum_{c=1}^C w_j^{c}  \hat g_j^{c,v(i)}(\underline{a}_j|W_i,A_{-ji})\right]^2, \nonumber \\ & \hspace{2em} \text{ subject to } \sum_{c=1}^C w_j^c = 1 \text{ and } w_j^c \ge 0 \text{ for all } c. \nonumber
\end{align}
One can obtain the $(w_j^{1*},...,w_j^{C*})$ weights vector by fitting a constrained non-negative least squares regression of the observed $I(A_{ji} = \underbar{a}_j)$ values on the estimated $(\hat g_j^{c,v(i)}(\cdot), ..., \hat g_j^{C,v(i)}(\cdot))$ values \citep{vanderlaan-etal2007-superlearner}.  Given these weights, we fit the candidate algorithms on the complete data,  and the ensemble prediction for the propensity score is given as,
$$
\hat g_j(\underline{a}_j|W_i,A_{-ji}) = \sum_{c=1}^C w^{c*}_j \hat g_j^c(\underline{a}_j|W_i,A_{-ji}).
$$
\citet[Thm. 1]{vanderlaan-etal2007-superlearner} show that under mild regularity conditions, the mean square error of prediction for $\hat g_j(\cdot)$ converges in $N_v$ to the mean square error of the best candidate algorithm.  Therefore the consistency properties of $\hat g_j(\cdot)$ are inherited from the best candidate algorithm. 

The candidate algorithms in our ensemble include the following: (i) logistic regression, (ii) $t$-regularized logistic regression \citep{gelman-etal2008-weakly-informative-prior}, (iii) kernel regularized least squares (KRLS) \citep{hainmueller-hazlett2012-krls}, (iv) Bayesian additive regression trees (BART) \citep{chipman-etal2010}, and (v) $\nu$-support vector machine classification (SVM) \citep[Ch. 12]{chen-etal2005-nu-svm, hastie-etal2009-elements}.  This ensemble includes methods that are demonstrably effective in hunting out nonlinearities (e.g., kernel regularized least squares and support vector classification) and interactions (e.g., Bayesian additive regression trees).\footnote{This ensemble represents the full set of algorithms for which the authors know of research demonstrating  effectiveness in relevant applied settings.  In using the approach developed in this paper, researchers are free to consider other, potentially superior algorithms in their ensemble.}  We use 10 cross validation splits ($V=10$ in our ensemble). \citet{polley-etal2011-superlearning} demonstrate that a 10-fold cross validation super learner using some of these algorithms (they do not include KRLS) performs well in a wide range of data settings,  including in estimating highly irregular and non-monotonic conditional mean functions.

In our illustration below, we use a rich covariate set, and so our ensemble relies primarily on regularized methods that reward sparsity  (that is, they shrink partial effects of covariates to zero) in order to further control over-fitting \citep{bickel-li2006-regularization}.  The importance of such regularization is likely to be important when the covariate set contains large amounts of noise that obscure identifying variation.  The only non-regularized method is logistic regression, which does not reward sparsity but is a method that we include because it remains the workhorse approach to propensity score estimation in political science. This provides a useful benchmark to evaluate gains from the much more computationally complicated algorithms and the ensemble routine overall, since we can view the weight given by the super learner to logistic regression relative to the other methods.  

The kernel regularized least squares, Bayesian additive regression trees, and $\nu$-support vector classification and regression algorithms are based on models that grow in complexity with the data,\footnote{Estimators that grown in complexity like this are known as ``sieve'' estimators \citep{geman-hwang1982-nonparametric-mle}.} although such growth is constrained by regularization parameters. In a manner similar to Taylor approximation, allowing for more complexity helps to ensure improved approximations and  consistency for the predicted mean conditional on the covariates included in the analysis \citep{greenshtein-ritov2004-pred-cons}.  

In our ensemble, we economize on computational costs by using the default rule-of-thumb settings for the regularization parameters that approximate MSE minimization.\footnote{The rule of thumb methods are specific for each algorithm.  See \citet[1364-165]{gelman-etal2008-weakly-informative-prior} for $t$-regularized logistic regression, \citet[pp. 6-7]{hainmueller-hazlett2012-krls} for KRLS, and \citet[269-273]{chipman-etal2010} for BART, and \citet[p. 129]{chalimourda-etal2004} for $\nu$-support vector classification.}  In principle, one could incorporate into the ensemble multiple versions of each algorithm, with each version applying a different regularization parameter, and then construct the cross-validated error-minimizing combination, although this could entail relatively high computational costs.

\section{Simulation study}

We provide evidence on finite sample performance of the ensemble method using a simulation study that illustrates that challenge of extracting meaningful variation in covariate sets as the noise-to-signal ratio increases.\footnote{For replication materials, see \cite{samii-2016-rep-files}.}  We consider a situation where we have observational data on an outcome $Y$, a single binary treatment variable $A=0,1$, and then a vector of covariates, $W$. Our estimand is the RIE for a hypothetical intervention that removes exposure to the treatment --- that is, it sets $A=0$ for everyone.  This corresponds to the case that we explored above in the decomposition that relates the RIE to the ATT.  The outcome $Y$ depends on the value of $A$ and underlying potential outcomes, $(Y(1), Y(0))$---that is, $Y=AY(1) + (1-A)Y(0)$. We set up the simulation so that outcomes and treatment assignment probabilities are a function of only one covariate, $W_1$:
\begin{align}
Y(0) & = W_1 + .5 (W_1 - \min(W_1))^2 + \epsilon_0 \nonumber \\
Y(1) & = W_1 + .75 (W_1 - \min(W_1))^2 + .75 (W_1 - \min(W_1))^3 +\epsilon_1 \nonumber \\
\Pr[A=1|W_1] & = \text{logit}^{-1}\left(-.5 + .75W_1 -.5[W_1 - \text{mean}(W_1)]^2\right),
\end{align}
where $\epsilon_0 \sim \N(0, 5^2)$, $\epsilon_0 \sim \N(0, 10^2)$, $W_1 \sim \N(0,1)$, and $\min(W_1)$ and $\text{mean}(W_1)$ take the minimum and mean, respectively, of the sample draws of $W_1$ prior to producing the $(A, Y(0), Y(1))$ values.\footnote{Using the minimum and mean in this way are simple ways to control how the non-linearity appears in the sample.} Figure \ref{fig:sim-graphs} displays data from an example simulation run.

\begin{figure}[!t]
\label{fig:sim-graphs}
\begin{center}
\includegraphics[width=.75\textwidth]{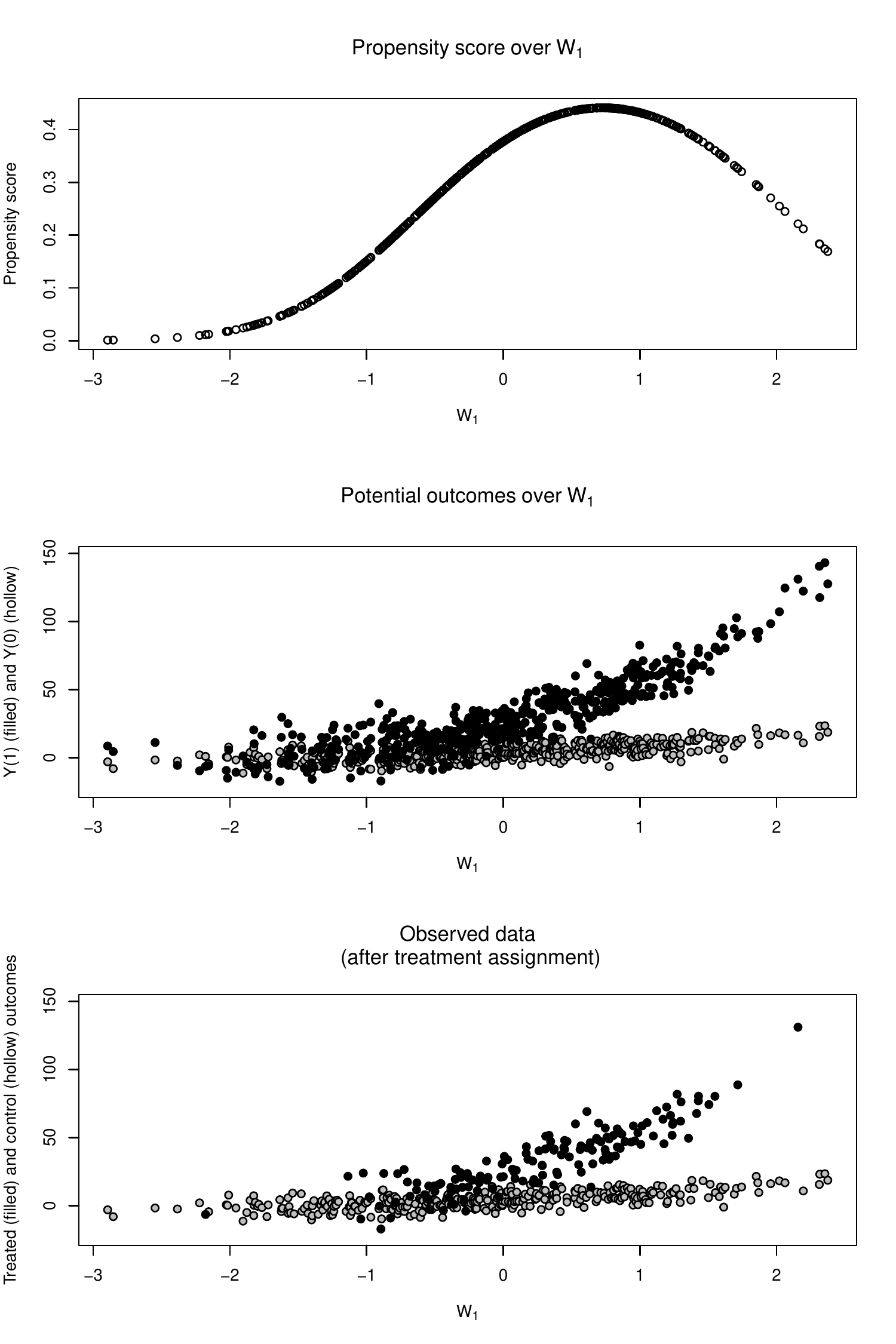}
\end{center}
\caption{Plots from an example simulation run. The top plot shows the expected value of the propensity score over the confounding covariate, $W_1$.  The middle plot shows potential outcomes under treatment (filled) and control (hollow) for the full sample.  The bottom plot shows observed outcomes for those assigned to treatment (solid) and control (hollow).}
\end{figure}%

One goal of the simulation is to show how our machine learning ensemble handles non-linear and non-monotonic functions such as the ones displayed in Figure \ref{fig:sim-graphs}.  Another goal is to study the challenge of working with a high dimensional covariate in which the identifying variation in $W_1$ is obscured by the existence of other covariates with little identifying power.  Therefore, in addition to working with just $W_1$, we add first 5 and then 10 dimensions of pure white noise to the covariate set---that is, 5 and then 10 additional covariates, each drawn  independently as $\N(0,1)$ and thus unrelated to either $Y$ or $A$.  We want to see how well various methods perform in sorting through all of this noise to extract the variation that is meaningful for causal identification.

In our study, we compare four methods to estimate the RIE:
\begin{enumerate}
\item Ordinary least squares (OLS) regression where we regress $Y$ on $W_1$ and then the other covariates, with no interactions or higher order terms, where the coefficient on $A$ serves as our estimate;
\item Naive inverse-propensity score weighting (IPW) where we first estimate the propensity score using a logistic regression of $A$ on $W_1$ and then the other covariates, with no interactions or higher order terms; then, we use the estimated propensity score to construct the RIE estimate; 
\item Mahalanobis distance nearest-neighbor matching with replacement on $W_1$ and the other covariates to construct the counterfactual quantities in the RIE expression and then combining them to compute the RIE; note that the Mahalanobis distance metric corresponds precisely to the joint normality of the covariates;
\item Ensemble IPW which first uses the machine learning ensemble that we described above to estimate the propensity score with $W_1$ and the rest of the covariates, and then uses the estimated propensity score to construct the RIE estimate.  
\end{enumerate}
The data generating process exhibits a combination of issues that complicate causal effect estimation in the real world: (1) effect heterogeneity, (2) non-linearities in the relationship between covariates and potential outcomes, (3) non-linearity in the relationship between covariates and propensity scores, and (4) covariates of differing value for determining assignment and outcomes.  The methods described above handle these issues differently, with consequences for expected bias.  The OLS estimator ignores all four of the issues.  The naive IPW estimator ignores non-linearity in the propensity score (issue 3) and the differing importance of covariates (issue 4).  The matching estimator ignores the differing importance of covariates (issue 4).  The ensemble IPW estimator attends, in principle, to all four issues.

\begin{figure}[!t]
\label{fig:sim-out}
\begin{center}
\includegraphics[width=.75\textwidth]{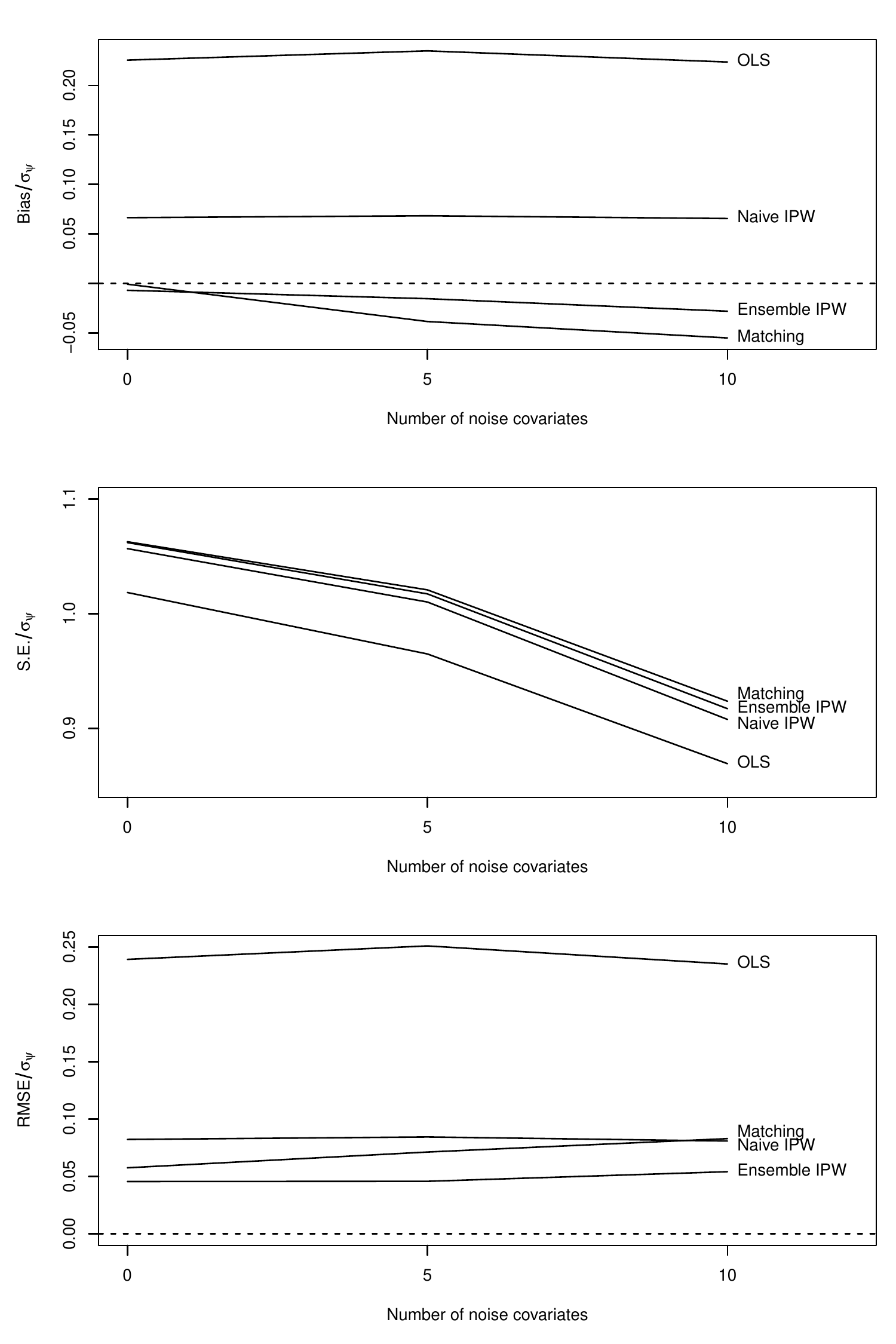}
\end{center}
\caption{Simulation results.  From top to bottom, the graphs show bias, standard error (S.E.), and root mean square error (RMSE) for the different estimators of the RIE from 250 simulation runs as the number of noise covariates increases from 0 to 10.  All results are standardized relative to the standard deviation of the true sample RIE across the simulation runs.}
\end{figure}%

Results from 250 simulation runs with a sample size of 500 are displayed in Figure \ref{fig:sim-out}.\footnote{The ensemble method is fairly slow to run because it employs ten-fold cross-validation, meaning that the simulations also run quite slowly. The results become quite stable after about 150 simulation runs; letting it run for 250 provided some extra security on convergence.}  The graphs display bias, the standard error (S.E.; that is, standard deviation of estimates across the simulation runs), and then root mean square error (RMSE) for 0 noise covariates, 5 noise covariates, and then 10 noise covariates.  These results are all standardized relative to the standard deviation of the true RIE over simulation runs ($\sigma_{\psi} = 3.60$).  In terms of bias, the OLS and naive IPW estimates are clearly poorest, owing to misspecification which for OLS fails to characterize the dramatically increasing effects in $W_1$ and for naive IPW fails to capture the peak in the propensity score.  The increase in noise covariates does not appreciably affect their biases.  With no noise covariates, matching and ensemble IPW are similarly unbiased.  Matching, however, is very sensitive to the increase in noise covariates.  The problem is that as we introduce more covariates, the meaningful differences (in terms of bias minimization) in $W_1$ are overwhelmed by meaningless differences in the other dimensions.  As a result, matches tend to become more random relative to $W_1$, and because of the way the data are distributed in the covariate space, we get negative bias.  The ensemble IPW estimator is much less sensitive to these problems---bias is half the magnitude when we get to 10 covariates.  All methods perform similarly in terms of their standard errors, with matching performing slightly worse than the rest.  RMSE combines these effects, showing that the ensemble IPW estimator is barely affected by higher dimensions of covariate noise.  By the time we get to 10 noise covariates, matching is performing as poorly (in an RMSE sense) as the misspecified naive IPW estimator. The misspecified OLS estimator is far and away the worst.   

The simulation captures the two reasons that we turn to machine learning ensembles.  First, the ensemble is effective in the presence of irregular functional forms and unlike OLS or naive IPW, we do not have to pre-specify these functional forms.  Second, the ensemble is not overwhelmed by noise in the covariate space the way that matching is.  Both estimators are consistent in terms of sample size for the RIE, but they differ in their finite sample performance depending on the amount of covariate noise.  Matching's performance degrades substantially even with 5 or 10 noise covariates.  In the application below, the number of covariates is much higher. 

\section{Application to anti-recidivism policies in Colombia}

Our application is to a study of policy alternatives to reduce recidivism among demobilized paramilitary and guerrilla fighters in Colombia.  By ``recidivism'' we refer to the committing of violent crimes such as murder, assault, extortion, or armed robbery after demobilization.  Such recidivism among former combatants is at the heart of the troubling emergence of ``{\it bandas criminales}'' that have taken charge of narcotics trafficking and threatened social order across Colombia \citep{icg2012-bacrim}. The analysis was meant to shed light on the kinds of interventions that might be most promising for the government to undertake to battle recidivism and increase former militants' reintegration into civilian life.  Of particular interest was how funds might be best allocated across potential interventions targeting economic welfare, security, relations with authorities, psychological health, and relations among excombatants. 

Our data are from a representative multistage sample of 1,158 ex-combatants fielded in 47 Colombian municipalities between November 2012 and May 2013 in collaboration with a Colombian think tank, {\it Fundaci\'on Ideas para la Paz}, the Colombian government department charged with the reintegration of former combatants (the {\it Agencia Colombiana para la Reintegraci\'on}), and the Organization of American States {\it Misi\'on de Apoyo al Proceso de Paz}.  The survey sought to achieve representativeness for the full population of demobilized combatants in Colombia, and included prisoners, ``hard to locate'' ex-combatants, as well as ex-combatants in good standing with the authorities.\footnote{Details on the methods that we used to construct the sample are given in [reference to unpublished work withheld].}  In addition to the survey responses for the individuals in the sample, we obtained a rich set of variables from administrative records of the Colombian attorney general's office ({\it Fiscalia General de la Naci\'on}) and government agencies in charge of ex-combatant reintegration programs.  

\begin{table}[!t]
\begin{center}
\caption{Risk factors and hypothetical interventions}
\label{tab:interventions}
\includegraphics[width=1\textwidth]{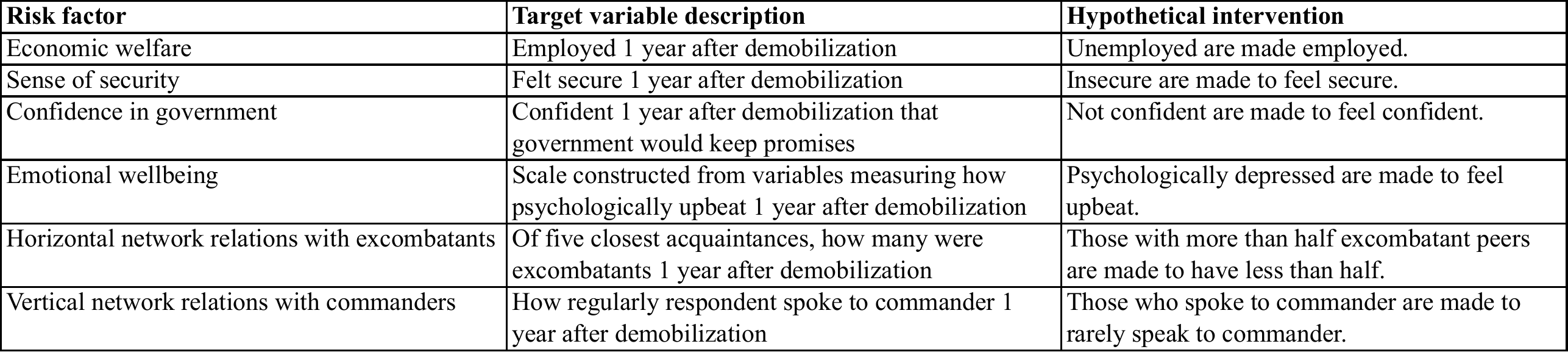}
\end{center}
\end{table}%

The first step of the analysis required that we define a set of risk factors and associated hypothetical interventions.  We defined these in consultation with relevant government authorities, establishing a list of six risk factors and associated hypothetical interventions.  These risk factors, associated variables, and hypothetical interventions are shown in Table 1. %\ref{tab:interventions}.  
In some cases, the nature of the intervention has a clear programmatic interpretation, such as ensuring that the ex-combatant is employed.  In other cases, the nature of the interventions is, admittedly, a bit vague.  For example, ensuring that excombatants have confidence in government at a level that is above 5 in a 10-point scale does not have an immediately actionable interpretation.  What we imagine is that there could be an intervention that generates such a change in attitudes.

Having established the risk factors and interventions, the next step was to establish a covariate set that would allow for credible causal identification.   Our covariate set includes data extracted from the administrative files, measures obtained through the surveys, and then municipality fixed effects, for a total of 114 covariates.  The covariates account for individuals' household, personal, and various contextual circumstances prior to joining their respective armed group, various facets of their experience during their time in the armed group, and the nature of their demobilization and reintegration experience.   To reduce measurement error, we performed a preliminary stage of dimension reduction using a one-factor latent trait analysis that reduced the dimensionality of our covariate set to a set of 23 indices constructed by taking inverse-covariance weighted averages of variables that can reasonably be assumed to capture common traits \citep{obrien84_multiple}. This preliminary step of dimension reduction was pre-specified prior to data collection, which established {\it ex ante} the sets of items that were meant to capture common traits.  The covariate set for our final analysis uses these 23 indices along with a vector of nine demographic traits and dummy variables for the 47 municipalities in which the subjects had demobilized, and so a total of 79 covariates.  

Having defined treatments and covariates, the last step in the data preparation was in defining and measuring outcomes.  Given the sensitive nature of recidivism outcomes, we constructed a ``recidivism vulnerability index.''  The index takes its highest value of 3 for known recidivists and values ranging between 0 and 2 on the basis of the number of clues that our data show suggesting that the respondent is vulnerable to being recidivist.  The index is based on information from attorney general records (history of arrest, charges, and imprisonment), responses to survey questions on crimes committed, responses to survey questions on the extent to which illegal behavior might be condoned, and responses to survey questions on exposure to opportunities in which crimes might be committed.  The latter three were obtained via a self-administered questionnaire answered in private, following best practice in the survey literature for sensitive questions \citep{TourangeauYan05}.  Proven recidivists were those identified as such through the attorney general data or who, in our survey, admitted to being recidivist.

Table \ref{tab:outcome} displays the distribution of the recidivism index in the population and for subpopulations defined on the basis of the intervention variables.  We estimate that the population is fairly evenly distributed over the recidivism index levels.  For the intervention variables, however, we see that in some cases the population is not divided into two equally sized groups.  For example, only 18\% of the population reports that they were without employment one year after demobilization, and so it is only for this 18\% that the hypothetical employment intervention would apply.  Similar circumstances hold for the shares of the population that are depressed, have a large fraction of excombatants in their social networks, or that continued to speak to their commanders.  That being the case, the potential for interventions on these variables to make a major impact is limited to some extent.  Only if the effects were very pronounced would the RIE be of substantial magnitude.  We stress that this is a feature, not a bug, of the RIE approach: it tells us what kinds of policies might have the largest return, all things considered.  This takes into account the possibility that that share of the population for which there is a particular ``problem'' may be quite small.  Table \ref{tab:outcome} also shows differences in the recidivism index values over the intervention variables.  We see pronounced differences for all but the employment variable.  Of course, these comparisons could be biased by confounding.  Our propensity score approach addresses this possibility. 

\begin{table}[!t]
\caption{Recidivism Vulnerability Index Outcome and Intervention Variables ($N$=1,158)}\label{tab:outcome}
\begin{center}
\begin{tabular}{rlm{.25in}m{.25in}m{.25in}m{.25in}cc}
\hline
& Recidivism Index Value$^a$ = & 0  & 1 & 2 & 3 & Mean & (S.E.)\\ 
& & \multicolumn{4}{l}{(\% in each category)} \\
\hline
i.& Unweighted full sample & 27 & 26 & 15 & 33 & 1.53 & (.04)\\ 
& Weighted full sample$^b$ & 28 & 31 & 16 & 23 & 1.38 & (.06) \\
\hline
ii.& Has employment = 0 (18\%) & 25& 37 & 14 & 23  & 1.35 & (.09)\\
& Has employment = 1 (82\%) & 29 & 29 & 16 & 26 & 1.39 & (.07)\\
\hline
iii. & Has security = 0 (39\%) & 23 & 25 & 22 & 20 & 1.60 & (.08)\\
& Has security = 1 (61\%) & 32 & 35 &  12 &  22 & 1.24 & (.07)\\
\hline
iv. & Confidence in govt. = 0 (42\%) & 16 & 29 & 23 & 32 & 1.70 & (.07)\\
& Confidence in govt. = 1 (58\%) & 37 & 33 & 11 & 20 & 1.15 & (.07)\\
\hline
v. & Not depressed = 0 (23\%) & 18 & 24 & 24 & 34& 1.74 & (.14) \\
& Not depressed = 1 (77\%) & 31 & 33 & 14 & 22 & 1.27 & (.06)\\
\hline
vi. & Few excom. peers = 0 (19\%)&21 & 26 & 18 & 35 & 1.67 & (.12)\\
& Few excom. peers = 1 (81\%) & 30 & 32 & 15 & 23& 1.31 & (.06)\\
\hline
vii. & Doesn't speak to commander = 0 (15\%)& 22 & 23 & 17 & 38 & 1.71 & (.13)\\
& Doesn't speak to commander = 1 (85\%) & 29 & 32 & 16 & 24 & 1.32 & (.06)\\
\hline
\multicolumn{6}{l}{\small $^a$0 = ``non-recidivist,'' 3=``proven recidivist.''}\\
\multicolumn{6}{l}{\small $^b$Incorporates survey weights to account for unequal sampling probabilities across}\\
\multicolumn{6}{l}{\small  sample strata.}\\
\multicolumn{6}{l}{\small  i.-vii. Multiple imputation estimates of sample proportions.}\\
\multicolumn{6}{l}{\small  ii.-vii. Estimates use sampling weights.}
\end{tabular}
\end{center}
\label{default}
\end{table}%

The survey data exhibited small amounts of item-level missingness on the various covariates, however such missingness adds up and would have resulted in dropping a non-negligible amount of data.  We used ten-round multiple imputation, with imputations produced via predictive mean matching \citep{royston2004-ice}.  Because of the low item-level missingness, the imputation method is unlikely to make much of a difference in the results, and predictive mean matching is robust to misspecification.  Estimates were constructed from the imputation-completed datasets using the usual combination rules, with point estimates computed as the mean of estimates across imputations and standard errors computed in a manner accounting for both the within- and between-imputation variances \citep[85-89]{little-rubin2002-missingdatabook}.  (Table \ref{tab:workflow} in the supplementary materials shows the workflow.) We fit the components of the ensemble using associated R (v.3.0.3) packages for each of the estimation methods.  These were then fed into the {\it SuperLearner} package for R \citep{polley-vanderlaan2012-superlearner-package} to perform the cross-validation and MSE-based averaging that produced our propensity score estimates.  Then, effects, standard errors, and confidence intervals were constructed based on our survey design with the {\it survey} package in R \citep{lumley2010-survey-book}.

Figure \ref{fig:wgts} shows the weights that the prediction methods received in the ensembles predicting the different intervention propensity scores.   Recall that for each intervention, the weights are obtained from a constrained regression of the observed treatment values on the propensity scores from each prediction method, with the constraint being that coefficients cannot be less than zero and that they must sum to one.  The figure shows the predictive performance of each method.  Logistic regression performs very poorly, receiving zero weight in all ensembles except for the one predicting the propensity score for having few excombatant peer relationships. The weight given to the other methods varies over interventions. BART very regularly receives high weight---indeed, it is the only method that receives positive weight in all interventions.  But BART's weight is surpassed for the employment and security intervention and essentially ties for first place for the excombatant peers intervention.  Understanding why one or another method tends to perform well for different prediction problems could be a useful avenue for further research.  But the main take-away point here is that no single method would have been as reliable as the ensemble for these six prediction problems.  

\begin{figure}[!t]
\includegraphics[width=1\textwidth]{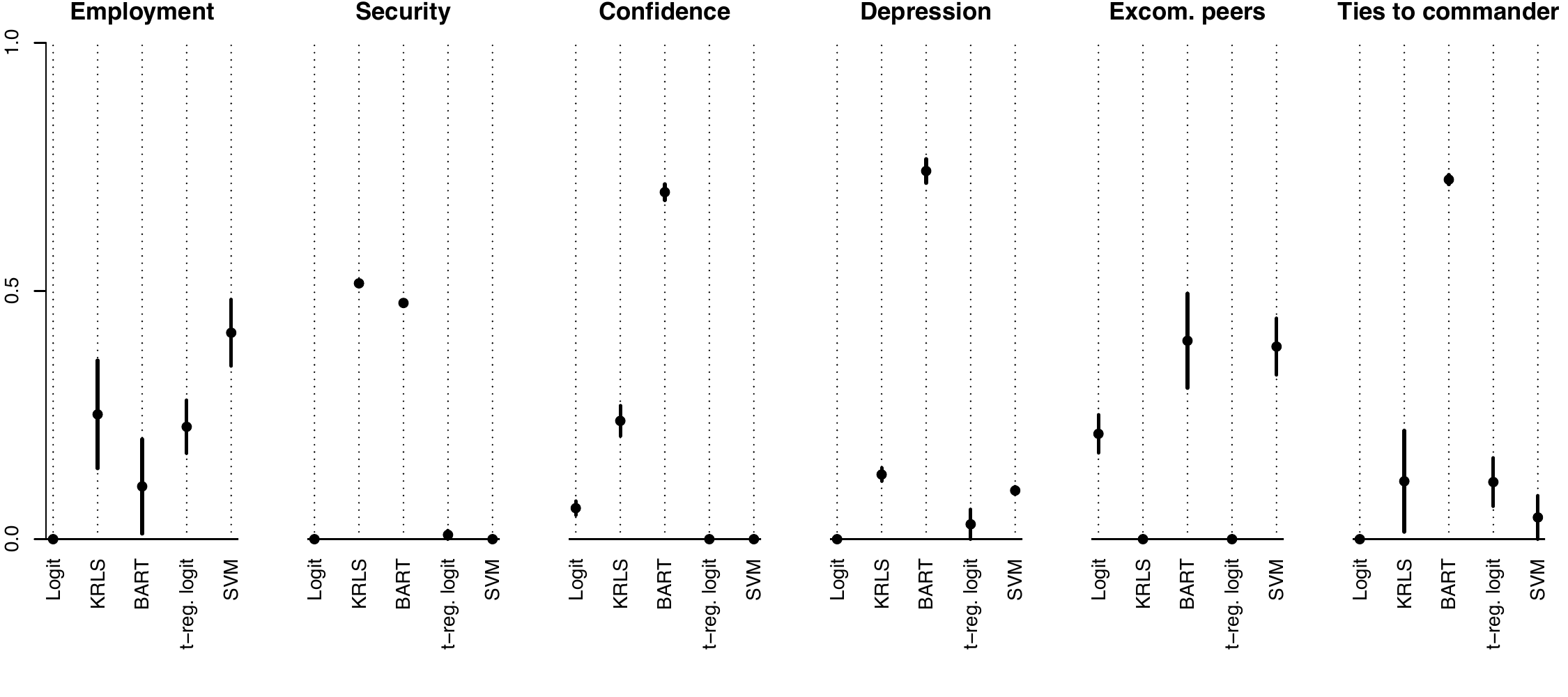}
\caption{Weights applied to propensity score predictions from each prediction method.  The values of weights run along the $y$-axis, and prediction methods run along the $x$-axis. Results are grouped by intervention. The weights are constrained to be no less than zero and to sum to one for each intervention. The black bars show the range of the weights over the 10 imputation runs, and the dots show the means.}
\label{fig:wgts}
\end{figure}

Figures \ref{fig:bal-pre} and \ref{fig:bal-post} demonstrate how the IPW adjustment removed confounding for estimating the RIEs.   Figure \ref{fig:bal-pre} shows the results of a placebo test that estimates pseudo-RIEs using {\it covariates} as outcome variables. Thicker horizontal bars are 90\% confidence intervals and thinner bars are 95\% intervals.  This plot allows us to see how the subpopulations that we use to form the counterfactual approximations differ from the overall population in terms of covariate means. The plot shows a high degree of imbalance.  If we did not reweight by the inverse of $\hat g_j(.)$, these covariate imbalances would confound the RIE estimates.  Figure \ref{fig:bal-post} shows that the IPW adjustment removes these mean differences and the potential for confounding.  A few covariates remain slightly out of balance in terms of their means, but no more than would be expected by chance (as evident from rates at which the confidence intervals fail to cover zero).

Figure \ref{fig:pscores} shows the distribution of propensity scores estimated by the ensemble for each intervention.  The histograms display the propensity scores of units for which $A_{ji}=\underline{a}_j$.  These are the units that are {\it not} subject to intervention and thus provide the outcome data used to construct the counterfactual mean for units that {\it are} subject to the interventions (that is, for whom $A_{ji}\ne \underline{a}_j$).  The propensity scores are clearly bounded away from zero, which is important for estimator stability.  In some cases propensity scores are very close to the value of 1, which is indicative of covariate combinations for which there would be few if any units subject to intervention in expectation.

\begin{figure}[!t]
\includegraphics[width=1\textwidth]{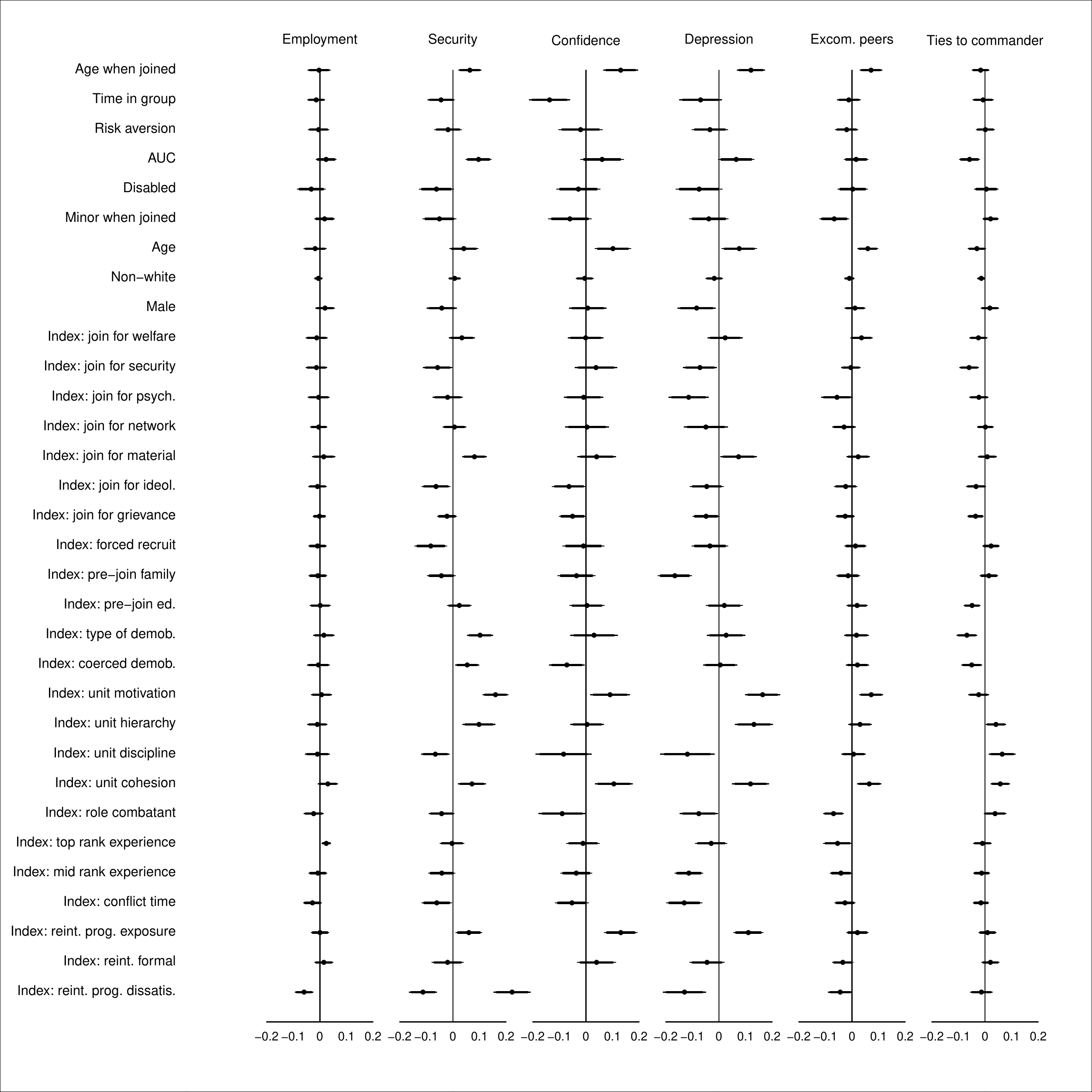}
\caption{Tests of mean balance for covariates and covariate indices in the raw data, prior to IPW adjustment.  Mean differences are shown in standard deviation units.  The horizontal bars passing through the points are the 95\% (thin) and 90\% (thicker) confidence intervals for the mean differences.}
\label{fig:bal-pre}
\end{figure}
\clearpage

\begin{figure}[!t]
\includegraphics[width=1\textwidth]{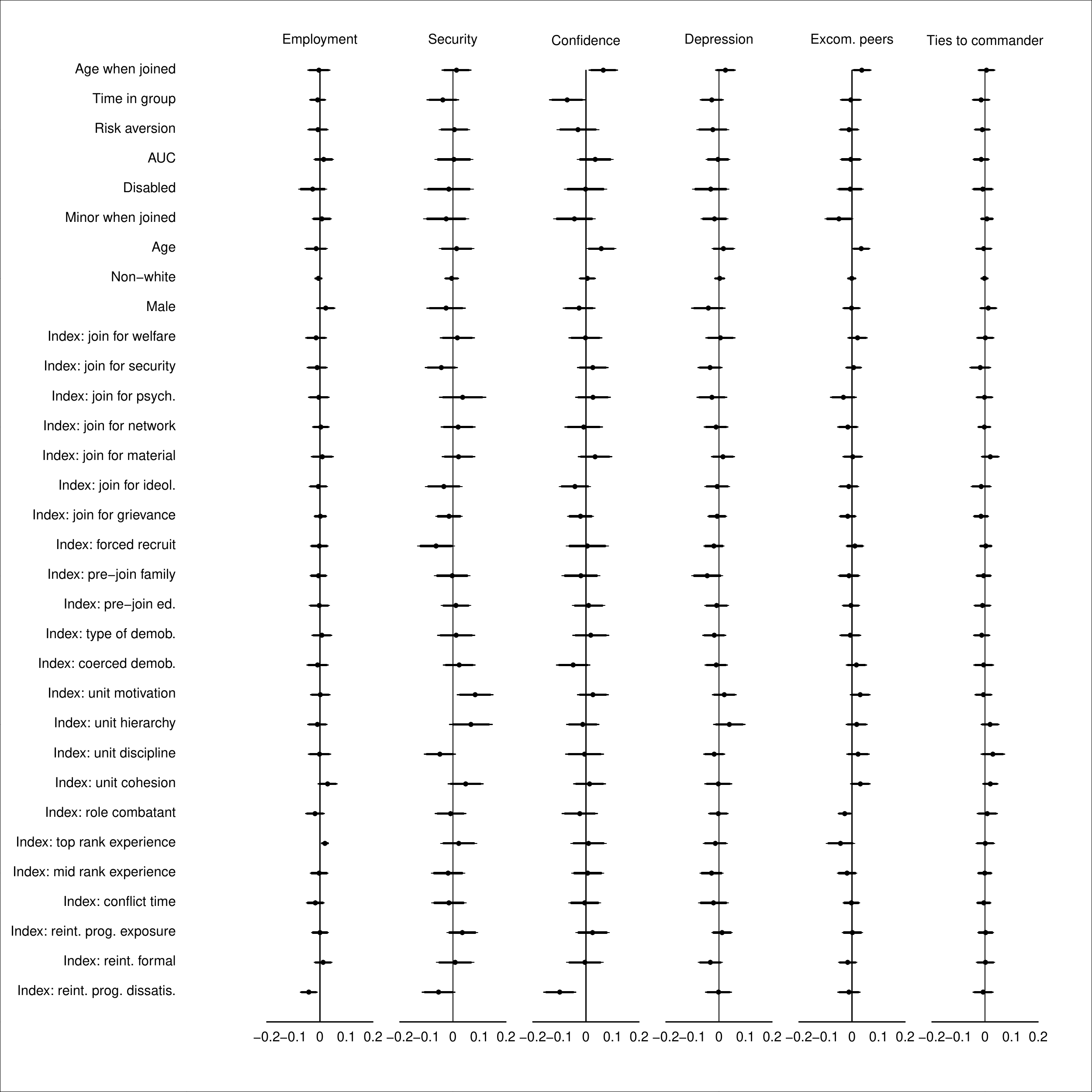}
\caption{Tests of mean balance for covariates and covariate indices with the IPW-adjusted data.   Mean differences are shown in standard deviation units.  The horizontal bars passing through the points are the 95\% (thin) and 90\% (thicker) confidence intervals for the mean differences.}
\label{fig:bal-post}
\end{figure}
\clearpage

\begin{figure}[!t]
\label{fig:pscore-hist}
\includegraphics[width=1\textwidth]{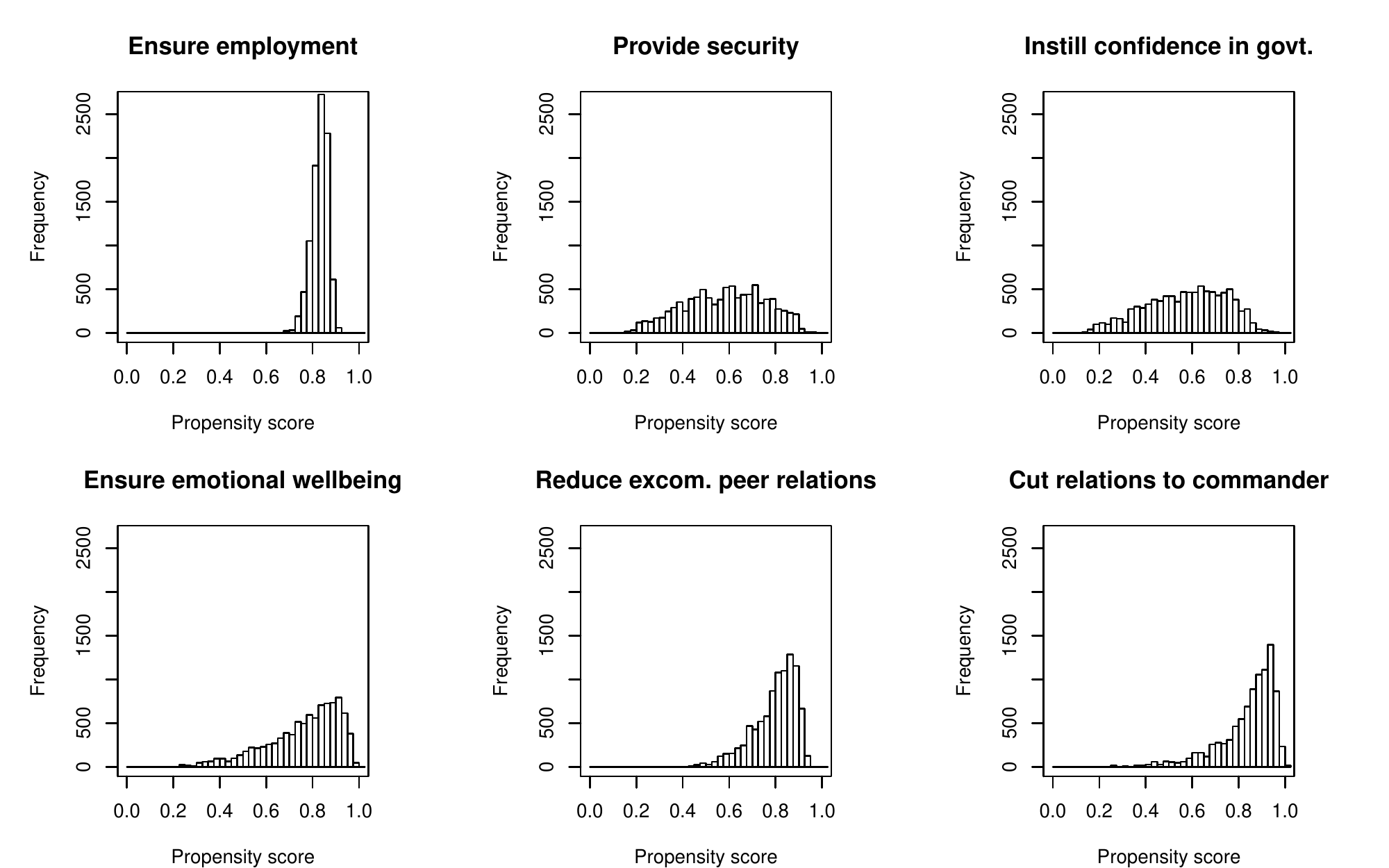}
\caption{Histograms of propensity scores estimated by the machine learning ensemble for each of the interventions.  The histograms show propensity scores for those not subject to the intervention, as they are the units used to construct the counterfactual outcome distribution for those who are subject to the intervention.}
\label{fig:pscores}
\end{figure}

Figure \ref{fig:results} plots RIE estimates and respective 95\% confidence intervals.  The figure displays the estimates based on the ensemble IPW method (black dots) and then estimates from the following comparison estimators: (i) a survey weighted least squares (WLS) regression, where the latter involved regressing the outcome on the hypothetical intervention variables and then on a control vector that included the 23 indices, demographic controls, and municipality fixed effects with no higher order terms of interactions; (ii) a matching estimator that uses one-to-one Mahalanobis distance nearest neighbors matching with replacement to construct the counterfactual mean for those who would be subject to the intervention, with exact matching on municipality indicators; and (iii) a naive IPW estimator that uses propensity scores from a logistic regression of the relevant treatment on a linear specification for the control variables.  

The different estimators yield similar findings in terms of the general direction of the various effects and the way the different interventions are ranked in terms of their beneficial effects (note that negative estimates are beneficial in this context).  Where the real differences lie are in the scale of the point estimates.  The ensemble IPW estimates are generally closer to zero than the WLS estimates, but generally further away from zero than the matching or naive IPW estimates.  In policy analysis, these scale differences are important, because cost effectiveness analyses depend on the point estimates. The WLS estimates seem to exaggerate the effects of different interventions, while the matching and naive IPW estimates seem to heavily understate them.

The RIE estimates are defined in terms of shifts in the population mean.  Recall from Table \ref{tab:outcome} that the population mean in the recidivism index is 1.38 with a standard deviation of 1.14. Thus, the ensemble IPW point estimate for what appears as the most promising intervention---an intervention that instills confidence in government---is estimated to have reduced the average of recidivism tendencies by about 0.3 on the scale of the index or about a quarter of a standard deviation.  That would be a very meaningful effect substantively.  Note that the scale of this effect is a product of both the magnitude of the effect as well as the extent to which such an intervention would require the altering of individuals' treatment values.  For this intervention, Table \ref{tab:outcome} showed that 42 percent of the excombatants had confidence index values below the intervention threshold,\footnote{The percentage is the same with and without the survey weights.} and so it is for them that the intervention would induce a counterfactual change. By contrast, the hypothetical employment, emotional wellbeing, excombatant social networks, and relations to commander interventions would introduce counterfactual changes for smaller fractions of the population.  For these interventions the potential for a substantial RIE would be more limited on this basis. Even as such, we still find statistically and substantively significant RIEs for all but the employment intervention.  This illustrates how the RIE is a population level effect estimate, combining average unit-level effects with information on who should be treated.  This yields a quantity that is immediately informative for policy.

\begin{figure}[!t]
\includegraphics[width=1\textwidth]{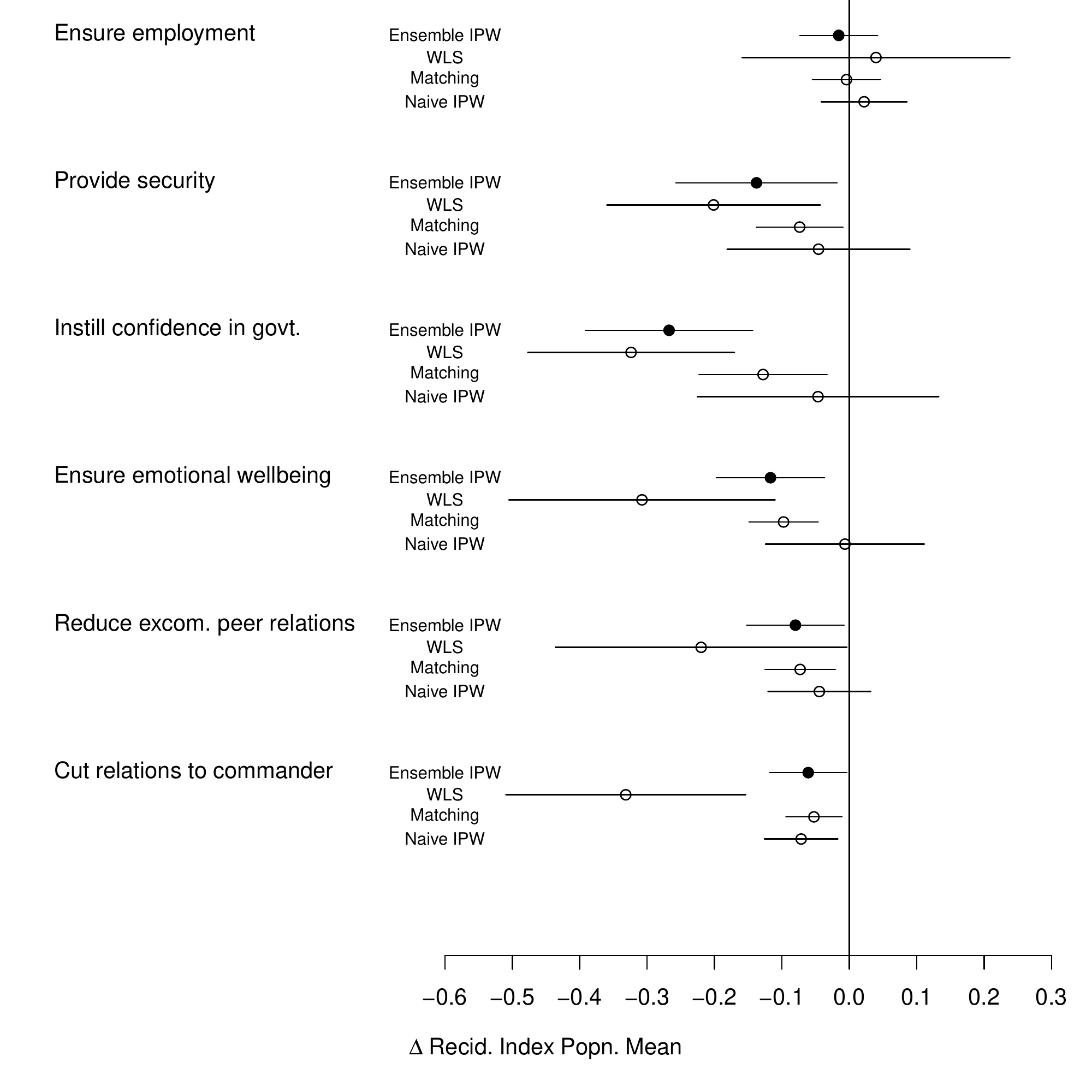}
\caption{Retrospective intervention effect estimates. The vertical line indicates the location of a null effect.  The plot shows point estimates (dots) and 95\% confidence intervals (horizontal bars running through the dots).  ``Ensemble IPW'' = inverse probability weighting RIE estimator, using the ensemble propensity scores; ``WLS'' = weighted least squares estimator based on a regression on the intervention variables and a simple linear specification for the covariates; ``Matching'' = nearest neighbor Mahalanobis distance matching RIE estimator; ``Naive IPW'' = inverse probability weighting RIE estimator, using propensity scores from a logistic regression with a simple linear specification for the covariates.}
\label{fig:results}
\end{figure}

\section{Conclusion}

This paper considers a method for retrospective causal inference that applies machine learning tools to sidestep problems with conventional approaches.  Our approach has two core feature that each confer benefits. First, we define the ``retrospective intervention effect'' (RIE). The RIE uses the device of hypothetical interventions to pin down clear population-level counterfactual comparisons. It also allows us to evaluate, in an easy-to-interpret manner, the relative importance of different risk factors and their effect on a population's outcomes.  Second, we use a machine learning ensemble to use a large number of control variables for causal identification.  A simulation experiment shows the robustness of the ensemble relative to conventional methods in extracting identifying variation from irregular functional relationships in a noisy covariate space. We reweight using predicted propensity scores to approximate the counterfactual defined under hypothetical interventions. This creates a contrast between what actually happened and an estimate of what might have been.  An application to anti-recidivism policies in Colombia led to crisp conclusions about the relative merits of interventions on ex-combatants' confidence in government, social networks, security, and emotions, as compared to other risk factors, such as employment.

The range of problems for which these methods can be applied are constrained by the three identifying assumptions: (i) treatment consistency/SUTVA, (ii) conditional independence, and (iii) positivity.  The machine learning element frees us from the specification assumptions that previous methods also require.  Treatment consistency and SUTVA can be established, in principle, by properly defining interventions and levels of analysis.  For example, if SUTVA is thought to be violated at a low level of aggregation (e.g., individuals), there may be the possibility of satisfying it when we operate at a higher level of aggregation.  Conditional independence can be made more believable if we measure a very large set of covariates.  For methods requiring specification decisions, this in itself creates enormous complications. We overcome this challenge by incorporating regularized methods into our machine learning ensemble. The positivity assumption requires that there exist, in the real world, units that exhibit the diversity in treatment variables and covariates needed to construct a counterfactual approximation for a hypothetical intervention \citep{king_zeng06_extreme}.  This assumption is perhaps the most restrictive.  In some cases it may be satisfied by redefining the target population \citep{crump-etal2006-goal-posts}.  But doing so sacrifices the population-level inference that motivated us in the first place.  As far as we understand, this is an unavoidable limitation for any observational method (and probably experimental too, given practical and ethical limitations on experimental subject pools).

Retrospective studies are a crucial first step in many research programs. They are essential for understanding causes of outcomes that are rare or that emerge only after many years.  This includes outcomes such as violence or institutional change.  Oftentimes the goal is to sort through a number of potential causal factors to identify points of intervention that should be prioritized for experimental or prospective studies.  The conventional approach for doing so in the social sciences relies on multiple regression, for example in conventional case-control studies \citep{king-zeng2002-case-control, korn-graubard1999-health-surveys-book}.  However, the validity of multiple regression estimates depends on homogeneity and model specification assumptions that cannot be defended in many instances, and especially so when the set of control variables is large.  When the number of necessary control variables is large, other estimation methods such as matching, propensity score, or prognostic score methods either require modeling assumptions or make inefficient use of identifying variation.  Under such circumstances, there is reason to be concerned about both bias and the potential for researcher discretion to undermine the validity of the analysis.  The methods presented here demonstrate ways toward more objective and reliable retrospective causal inference. 

The machine learning ensemble allows the researcher to address the bewildering specification challenges that arise when working with a large number of covariates.   Having a large number of covariates at one's disposal allows, in principle, for more plausible causal identification under the conditional independence assumption.  At the same time, it raises concerns about researchers selecting from among the vast number of potential specifications to manipulate results.  The ensemble method can assuage such concerns in that it targets an objective criterion---the minimum expected error of prediction for the propensity score.  This limits researcher degrees of freedom in the specification search, although it does not remove them entirely.  The researcher still selects the algorithms, tuning parameters, loss functions, and preprocessing steps.  Good faith is still required for credible inference.

\clearpage
\bibliographystyle{pa}

%\bibliography{/users/nlsamii/Dropbox/COLOMBIA/WRITING/Colombia}
%\bibliography{/users/cyrussamii/Dropbox/COLOMBIA/WRITING/Colombia}

\clearpage
\singlespace
\begin{center}
{\LARGE Supplementary Materials for\\``Retrospective Causal Inference with Machine Learning Ensembles: An Application to Anti-Recidivism Policies in Colombia''}
\end{center}

\setcounter{page}{1}

\pagestyle{plain}
\appendix

\vspace{1in}
\section{Estimation and inference details}

\begin{prop}[Consistency]
Suppose we have 
\begin{itemize} 
\item a random sample of size $N$ of observations of $O$,
\item bounded support for $O$, 
\item Assumptions 1-3, and 
\item  $\hat{g}_j(\underline{a}_j|W_i,A_{-ji})$ a consistent estimator of $\Pr[A_j = \underline{a}_j|W_i,A_{-ji}]$.
\end{itemize}
Under such conditions, $\hat \psi^{IPW}_{j} -\psi_{j}  \overset{p}{\rightarrow} 0$ as $N \rightarrow \infty$. 
\end{prop}
\begin{proof}
By Chebychev's inequality, consistency follows from asymptotic unbiasedness and variance converging to zero for the estimator \citep[Thm. 2.1.1]{lehmann1999-elements}. By random sampling, Slutsky's theorem, consistency for $\hat{g}_j(\underline{a}_j|W_i,A_{-ji})$, and Assumption 1, as $N \rightarrow \infty$, $\hat \psi^{IPW}_{j}$ has the same convergence limit as 
\begin{align}
\bar \psi^{IPW}_{j}  = \frac{1}{N}\sum_{i=1}^N\frac{I(A_{ji}=\underline{a}_j)}{\Pr[A_j = \underline{a}_j|W_i,A_{-ji}]}Y_{i}(\underline{a}_j, A_{-j}) - \E[Y]. \nonumber
\end{align}
Then,
\begin{align}
\E[\bar \psi^{IPW}_{j}] & = \frac{1}{N}\sum_{i=1}^N\E\left[ \frac{\E[I(A_{ji}=\underline{a}_j)|W_i,A_{-ji}]}{\Pr[A_j = \underline{a}_j|W_i,A_{-ji}]}\E[Y_{i}(\underline{a}_j, A_{-j})|W_i,A_{-ji}]\right] - \E[Y] \nonumber\\
& = \E[Y(\underline{a}_j, A_{-j})] - \E[Y] = \psi_j, \nonumber
\end{align}
and so $\E[\hat \psi^{IPW}_{j}- \psi_j] \rightarrow 0$ as $N \rightarrow \infty$, establishing asymptotic unbiasedness.  Next, by consistency for $\hat{g}_j(\underline{a}_j|W_i,A_{-ji})$ and Slutsky's Theorem, $\Var[N \hat \psi^{IPW}_{j}]$ has the same limit as $\Var[N \bar \psi^{IPW}_{j}]$, and by random sampling and bounded support, 
$$
\frac{1}{N^2}\Var[N \bar \psi^{IPW}_{j}] = \frac{1}{N^2}\sum_{i=1}^N \Var\left[\frac{I(A_{ji}=\underline{a}_j)}{\Pr[A_j = \underline{a}_j|W_i,A_{-ji}]}Y_{i}(\underline{a}_j, A_{-j}) \right] \le \frac{c^2}{N}
$$  
for some constant $c$, in which case $\Var[\hat \psi^{IPW}_{j}] \rightarrow 0$ as $N \rightarrow \infty$, establishing that the variance converges to zero.  
\end{proof}

To construct confidence intervals, we rely on well-known results for sieve-type IPW estimators \citep{hirano-etal2003-ipw, hubbard-vanderlaan2008-pop-int}. Define
$$
D_{i,IPW} =  \left( \frac{I(A_{ji}=\underline{a}_j)}{\hat{g}_j(\underline{a}_j|W_i,A_{-ji})} - 1\right) Y_{i},
$$
in which case $\hat \psi^{IPW}_{j} = \frac{1}{N}\sum_{i=1}^N D_{i,IPW}$.

Suppose that $g_j(\underline{a}_j|W_i,A_{-ji})$ parameterizes the true distribution for $A_j$, and $\hat{g}_j(\underline{a}_j|W_i,A_{-ji})$ approaches the maximum likelihood estimate for $g_j(\underline{a}_j|W_i,A_{-ji})$. Then, $\hat \psi^{IPW}_{j,k}$ is asymptotically normal and the following estimator is conservative in expectation for the asymptotic variance:
$$
\hat{V}(\hat \psi^{IPW}_{j,k}) = \frac{v(D_{ki,IPW})}{N},
$$
where the $v(.)$ operator computes the sample variance.  Define $\hat S_{IPW} = \sqrt{\hat{V}(\hat \psi^{IPW}_{j,k})}$.  Then we have the following approximate $100\%*(1-\alpha)$ Wald-type confidence interval for our estimate:
$$
\hat \psi^{IPW}_{j,k} \pm z_{\alpha/2} \hat S_{IPW}.
$$

We can modify the estimation and inference procedure to account for non-i.i.d. data.  We have assumed that $(W,A_{-ji})$ is a sufficient conditioning set for causal identification and that the model for $g_j(.)$ is sufficient for characterizing  counter-factual intervention probabilities conditional on $(W,A_{-ji})$.  For this reason, non-i.i.d. data on $O$ do not require that we change anything about how we go about estimating $\hat g_j$.  However, we will have to account for any systematic differences between our sample and target population in the distribution of $(W,A_{-ji})$ when computing $\hat \psi^{IPW}_{j,k}$.  This estimator is consistent for $\psi^{IPW}_{j,k}$ only if it marginalizes over the $(W,A_{-ji})$ distribution in the population.  The solution is to apply sampling weights that account for sample units' selection probabilities \citep[Ch. 6]{thompson2012-sampling}.  When units' selection probabilities are known exactly based on a sampling design (as is the case in our application), we merely need to modify the expression for $\hat \psi^{IPW}_{j,k}$ to take the form of a survey weighted mean rather than a simple arithmetic mean.  Our standard error and confidence interval estimates apply the usual survey corrections for clustering and stratification in sampling design \citep[Ch. 11-12]{thompson2012-sampling}.

\clearpage
\section{Details on the application}

\begin{table}[!h]
\caption{Risk factors and hypothetical interventions, details}\label{tab:interventions}
\begin{center}
\includegraphics[width=1\textwidth]{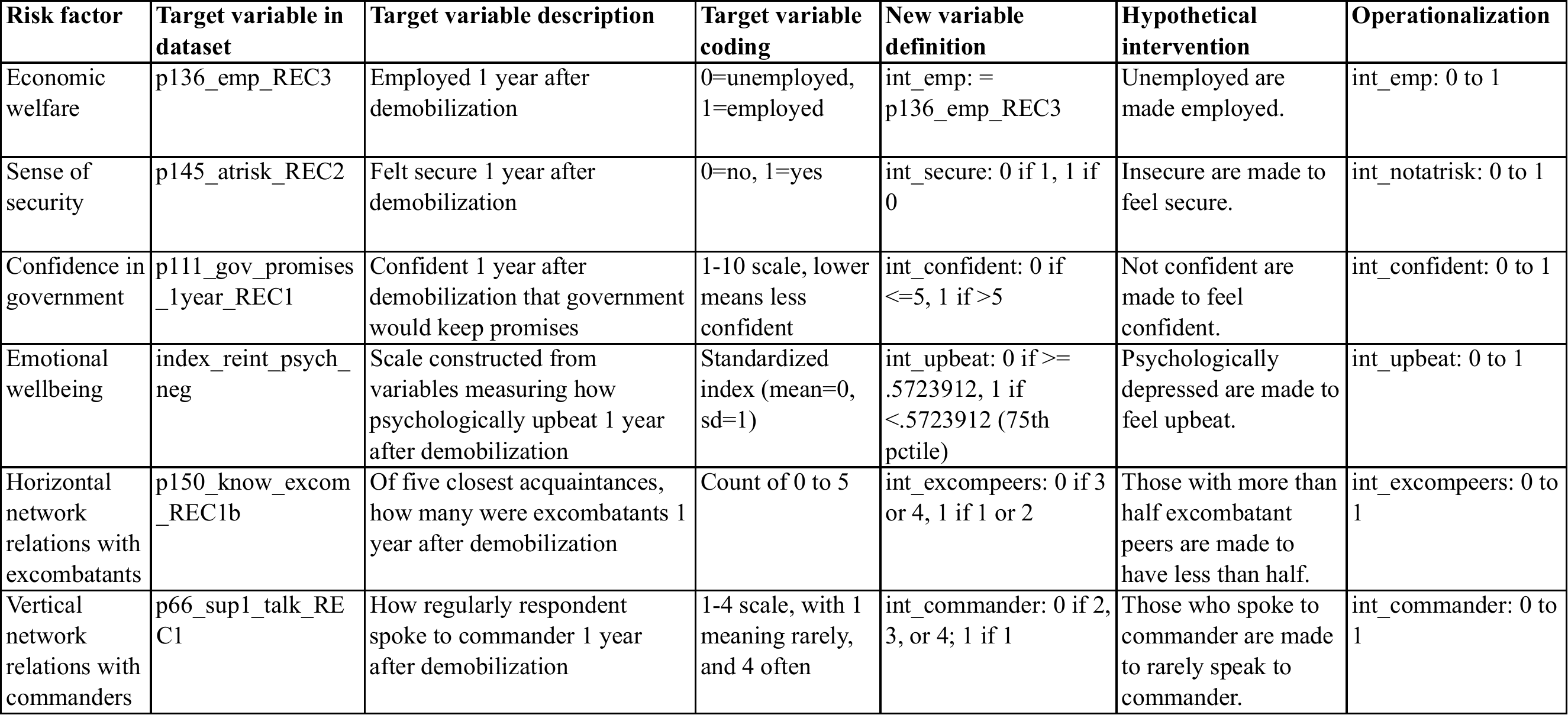}
\end{center}
\label{default}
\end{table}%

\clearpage

\begin{table}[!t]
\caption{Workflow for estimating RIEs with ensemble}
\begin{center}
\begin{tabular}{lp{2.5in}l}
{\bf Step} & {\bf Description} & {\bf Files}\\
\hline
1 & Define hypothetical interventions and construct intervention indicator variables; can be done in any software package. (Done on each imputation-completed dataset.)  &  \specialcell[t]{{\tt Hypothetical-Interventions.xlsx}\\ {\tt COLOMBIA-STEP9-interventions.do}}\\
\hline
2a & Fit propensity score models for each intervention with the ensemble, using cross-validated risk to generate optimal weights for the different model predictions; steps are automated with the SuperLearner functions for R. (Done on each imputation-completed dataset.)  & {\tt interv-pscore-1.R} through {\tt interv-pscore-6.R} \\
\hline 
2b & Generate predictions from propensity score models and attach to dataset.  Done using prediction functions in the SuperLearner package for R. (Done on each imputation-completed dataset.)& {\tt interv-pscore-1.R} through {\tt interv-pscore-6.R} \\
\hline
2c & Produce estimates of intervention effects, incorporating survey sampling adjustments; can be done with any survey estimation software, such as the survey package in R. (Done on each imputation-completed dataset, and then RIE estimates from the imputation-completed datasets were combined to obtain the final estimates.)  & {\tt interv-pscore-1.R} through {\tt interv-pscore-6.R}\\
\hline
3 & Summarize results.  & \specialcell[t]{{\tt int-results-graph.R}\\ {\tt int-results-balance-tables.R}\\ {\tt int-results-performance-metrics.R}}\\
\hline

\end{tabular}
\end{center}

\label{tab:workflow}
\end{table}%

\clearpage
\bibliographystyle{apsr}

\end{document}